\documentclass[10pt,twocolumn,letterpaper]{article}

\usepackage{cvpr}
\usepackage{times}
\usepackage{epsfig}
\usepackage{graphicx}
\usepackage{amsmath}
\usepackage{amssymb}
\usepackage[nice]{nicefrac}

\usepackage{wrapfig}
\usepackage{subfig}
\usepackage{multirow}
\usepackage{rotating}
\usepackage{microtype}
\usepackage{booktabs}
\usepackage[standard]{ntheorem}
\usepackage{tikz}
    \tikzstyle{vertex}=[circle, draw=black, fill=white, inner sep=0ex, minimum size=1ex]
    \tikzstyle{label-vertex}=[minimum size=1ex, inner sep=0ex, circle]
	\tikzstyle{cut-edge}=[dotted]
    \usetikzlibrary{decorations.pathmorphing}
    \tikzset{every picture/.append style={baseline,scale=1.1}}
	\definecolor{mycolor1}{HTML}{009900}
    \definecolor{g}{HTML}{009900}
    \definecolor{r}{HTML}{990000}
    \definecolor{b}{HTML}{000099}
    \tikzstyle{cut-edge}=[dotted]
    \tikzstyle{vertex}=[circle, draw, fill=white, inner sep=0pt, minimum width=1ex]
	\usetikzlibrary{arrows}
	\usetikzlibrary{decorations.markings}
	\tikzset{
		latex-arrow/.style={
			decoration={markings,mark=at position 1 with {\arrow[scale=1.5,#1]{latex}}},
    			postaction={decorate},
    			shorten >=0.4pt
		},
  		latex-arrow/.default=black
	}
	\tikzset{
		latex-arrow-red/.style={
			decoration={markings,mark=at position 1 with {\arrow[scale=1.5,#1]{latex}}},
    			postaction={decorate},
    			shorten >=0.4pt
		},
  		latex-arrow-red/.default=mycolor1
	}
	\tikzstyle{dotnode}=[circle,fill=black,inner sep=0ex,minimum size=1.5ex]
	\tikzstyle{mynode}=[circle, draw=black, fill=white, inner sep=0pt, minimum size=2ex]
	\tikzstyle{edge}=[draw=black,latex-arrow]
	\tikzstyle{red-edge}=[draw=mycolor1,latex-arrow-red]
\usepackage{pgfplots}
\usepackage{float}
    \floatstyle{plaintop}
    \newfloat{algorithm}{thp}{lop}
        \floatname{algorithm}{Algorithm}
    \newfloat{function}{thp}{lop}
        \floatname{function}{Function}
\usepackage{xcolor}
    
\usepackage{comment}
    \newcommand{\bst}[1]{$\mathbf{#1}$}
    \newcommand{\mapr}{\text{AP}}
\usepackage{rotating}
    \newcommand{\rotatedlabel}[1]{\begin{sideways}#1\end{sideways}}

\newcommand{\uiqp}{\textsc{uiqp}}
\newcommand{\lmp}{\textsc{lmp}}
\newcommand{\nllmp}{\textsc{nl-lmp}}

\newcommand{\kw}[1]{\textit{\textbf{#1}}}
\newcommand{\hw}[1]{{\color{mycolor1}#1}}
\newcommand{\ind}{\hspace{3ex}}
\newcommand{\alga}{KLj/r}
\newcommand{\algb}{KLj$*$r}

\usepackage[pagebackref=true,breaklinks=true,letterpaper=true,colorlinks,allcolors=mycolor1,bookmarks=false]{hyperref}
\cvprfinalcopy 

\def\cvprPaperID{xxxx} 

\begin{document}


\title{\mbox{Joint Graph Decomposition \& Node Labeling: Problem, Algorithms, Applications}\vspace{-3ex}}

\author{\begin{minipage}{\textwidth}
    \centering
    Evgeny Levinkov$^1$,
    Jonas Uhrig$^{2,3}$,
    Siyu Tang$^1$,
    Mohamed Omran$^1$,
    Eldar Insafutdinov$^1$,\\
    Alexander Kirillov$^4$,
    Carsten Rother$^4$,
    Thomas Brox$^2$,
    Bernt Schiele$^1$ and
    Bjoern Andres$^1$\\[1ex]    
    \small
    $^1$Max Planck Institute for Informatics, Saarland Informatics Campus, Saarbr\"ucken, Germany\\
    $^2$Computer Vision Lab, University of Freiburg, Germany \quad
    $^3$Daimler AG R\&D, Sindelfingen, Germany\\
    $^4$Computer Vision Lab, Technische Universit\"at Dresden, Germany
\end{minipage}}
\maketitle

\begin{abstract}
We state a combinatorial optimization problem whose feasible solutions define both a decomposition and a node labeling of a given graph.
This problem offers a common mathematical abstraction of seemingly unrelated computer vision tasks, including instance-separating semantic segmentation, articulated human body pose estimation and multiple object tracking.
Conceptually, the problem we state generalizes the unconstrained integer quadratic program and the minimum cost lifted multicut problem, both of which are \textsc{np}-hard.
In order to find feasible solutions efficiently, we define two local search algorithms that converge monotonously to a local optimum, offering a feasible solution at any time.
To demonstrate their effectiveness in tackling computer vision tasks, we apply these algorithms to instances of the problem that we construct from published data, using published algorithms.
We report state-of-the-art application-specific accuracy for the three above-mentioned applications.
\end{abstract}


\section{Introduction and Related Work}

Graphs are a ubiquitous structure in the field of computer vision.
In this article, we state an optimization problem whose feasible solutions define both a decomposition and a node labeling of a given graph (Fig.~\ref{figure:teaser}).
We define and implement two local search algorithms for this problem that converge monotonously to a local optimum.
The problem that we state is abstract enough to specialize to seemingly unrelated computer vision tasks.
This abstraction allows us to apply the algorithms we define, without changes, to three distinct computer vision tasks: multiple object tracking, instance-separating semantic segmentation and articulated human body pose estimation.
We report state-of-the-art application-specific accuracy for these three applications.

\emph{Multiple object tracking} \cite{Berclaz:2011:KSP, Brendel:2011:MOT, Choi15, Fagot-Bouquet2016, kim_ICCV2015_MHTR, Milan:2014:CEM, Pirsiavash:2011:GOG, xiang2015learning, Zamir:2012:GMC} can be seen as a task requiring two classes of decisions:
For every point in an image, we need to decide whether this point depicts an object or background.
For every pair of points that depict objects, we need to decide if the object is the same.
Tang et al.~\cite{tang-2015,tang-2016} abstract this task as a graph decomposition and node labeling problem w.r.t.~a finite graph whose nodes are bounding boxes, and w.r.t.~01-labels indicating that a bounding box depicts an object.
We generalize their problem to more labels and more complex objective functions.
By applying this generalization to the data of Tang et al.~\cite{tang-2016}, we obtain more accurate tracks for the multiple object tracking benchmark \cite{MilanL0RS16} than any published work.

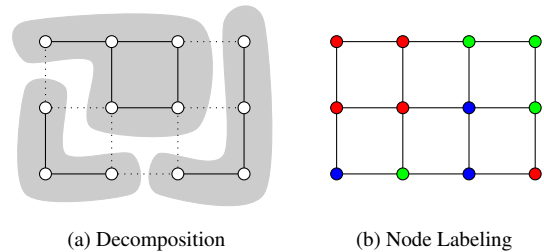
\begin{figure}
\centering
\subfloat[Decomposition]{%
\begin{tikzpicture}[scale=0.8]
    \draw[draw=black!20, fill=black!20] plot[smooth cycle, tension=0.5] coordinates
        {(-0.3, 2.3) (2.3, 2.3) (2.3, 0.7) (0.7, 0.7) (0.7, 1.7) (-0.3, 1.7)};
    \draw[draw=black!20, fill=black!20] plot[smooth cycle, tension=0.5] coordinates
        {(-0.3, -0.3) (-0.3, 1.3) (0.3, 1.3) (0.3, 0.3) (1.3, 0.3) (1.3, -0.3)};
    \draw[draw=black!20, fill=black!20] plot[smooth cycle, tension=0.5] coordinates
        {(1.7, -0.3) (1.7, 0.3) (2.7, 0.3) (2.7, 2.3) (3.3, 2.3) (3.3, -0.3)};
	\draw (0, 0) -- (0, 1);
	\draw (0, 0) -- (1, 0);
	\draw (0, 2) -- (1, 2);
	\draw (1, 1) -- (1, 2);
	\draw (1, 1) -- (2, 1);
	\draw (1, 2) -- (2, 2);
	\draw (2, 1) -- (2, 2);
	\draw (2, 0) -- (3, 0);
	\draw (3, 0) -- (3, 1);
	\draw (3, 1) -- (3, 2);
	\draw[style=cut-edge] (0, 1) -- (0, 2);
	\draw[style=cut-edge] (0, 1) -- (1, 1);
	\draw[style=cut-edge] (1, 0) -- (1, 1);
	\draw[style=cut-edge] (1, 0) -- (2, 0);
	\draw[style=cut-edge] (2, 0) -- (2, 1);
	\draw[style=cut-edge] (2, 1) -- (3, 1);
	\draw[style=cut-edge] (2, 2) -- (3, 2);
	\node[style=vertex] at (0, 0) {};
	\node[style=vertex] at (1, 0) {};
	\node[style=vertex] at (0, 1) {};
	\node[style=vertex] at (0, 2) {};
	\node[style=vertex] at (1, 1) {};
	\node[style=vertex] at (1, 2) {};
	\node[style=vertex] at (2, 1) {};
	\node[style=vertex] at (2, 2) {};
	\node[style=vertex] at (2, 0) {};
	\node[style=vertex] at (3, 0) {};
	\node[style=vertex] at (3, 1) {};
	\node[style=vertex] at (3, 2) {};
\end{tikzpicture}
\label{figure:decomposition}
}\subfloat[Node Labeling]{%
\begin{tikzpicture}[scale=0.8]
    \phantom{
        \draw[draw=green!30, fill=green!30] plot[smooth cycle, tension=0.5] coordinates
            {(-0.3, 2.3) (2.3, 2.3) (2.3, 0.7) (0.7, 0.7) (0.7, 1.7) (-0.3, 1.7)};
        \draw[draw=green!30, fill=green!30] plot[smooth cycle, tension=0.5] coordinates
            {(-0.3, -0.3) (-0.3, 1.3) (0.3, 1.3) (0.3, 0.3) (1.3, 0.3) (1.3, -0.3)};
        \draw[draw=green!30, fill=green!30] plot[smooth cycle, tension=0.5] coordinates
            {(1.7, -0.3) (1.7, 0.3) (2.7, 0.3) (2.7, 2.3) (3.3, 2.3) (3.3, -0.3)};
    }
	\draw (0, 0) -- (0, 1);
	\draw (0, 0) -- (1, 0);
	\draw (0, 2) -- (1, 2);
	\draw (1, 1) -- (1, 2);
	\draw (1, 1) -- (2, 1);
	\draw (1, 2) -- (2, 2);
	\draw (2, 1) -- (2, 2);
	\draw (2, 0) -- (3, 0);
	\draw (3, 0) -- (3, 1);
	\draw (3, 1) -- (3, 2);
	\draw (0, 1) -- (0, 2);
	\draw (0, 1) -- (1, 1);
	\draw (1, 0) -- (1, 1);
	\draw (1, 0) -- (2, 0);
	\draw (2, 0) -- (2, 1);
	\draw (2, 1) -- (3, 1);
	\draw (2, 2) -- (3, 2);
	\node[style=vertex, fill=blue] at (0, 0) {};
	\node[style=vertex, fill=green] at (1, 0) {};
	\node[style=vertex, fill=red] at (0, 1) {};
	\node[style=vertex, fill=red] at (0, 2) {};
	\node[style=vertex, fill=red] at (1, 1) {};
	\node[style=vertex, fill=red] at (1, 2) {};
	\node[style=vertex, fill=blue] at (2, 1) {};
	\node[style=vertex, fill=green] at (2, 2) {};
	\node[style=vertex, fill=blue] at (2, 0) {};
	\node[style=vertex, fill=red] at (3, 0) {};
	\node[style=vertex, fill=green] at (3, 1) {};
	\node[style=vertex, fill=green] at (3, 2) {};
\end{tikzpicture}
\label{figure:node-labeling}
}
\caption{This article studies an optimization problem whose feasible solutions define both a decomposition \protect\subref{figure:decomposition} and a node labeling \protect\subref{figure:node-labeling} of a given graph $G = (V,E)$.
A decomposition of $G$ is a partition $\Pi$ of the node set $V$ such that, for every $V' \in \Pi$, the subgraph of $G$ induced by $V'$ is connected.
A node labeling of $G$ is a map $f: V \to L$ from its node set $V$ to a finite, non-empty set $L$ of labels.\vspace{-9pt}}
\label{figure:teaser}
\end{figure}

\emph{Instance-separating semantic segmentation} \cite{cityscapes16cvpr, Dai2015, Liang2015, ren16arxiv, romera-2016, Ronneberger2015, zhang2016cvpr, zhang2015iccv} can be seen as a task requiring two classes of decisions:
To every point in an image, we need to assign a label that identifies a class of objects (e.g., human, car, bicycle).
For every pair of points of the same class, we need to decide if the object is the same.
Kroeger et al.~\cite{kroeger-2014} state this problem as a multi-terminal cut problem w.r.t.~a (super)pixel adjacency graph of the image.
We generalize their problem to larger feasible sets.
While Kroeger et al.~\cite{kroeger-2014} show qualitative results, we apply our algorithms to instances of the problem from the KITTI \cite{kitti12cvpr} and Cityscapes \cite{cityscapes16cvpr} benchmarks, obtaining more accurate results for Cityscapes than any published work.

\emph{Articulated human body pose estimation} can be seen as a task requiring two classes of decisions:
For every point in an image, we need to decide whether it depicts a part of the human body.
For every pair of points that depict body parts, we need to decide if they belong to the same body.
Pishchulin et al.~\cite{pishchulin-2016} and Insafutdinov et al.~\cite{insafutdinov-2016} abstract this problem as a graph decomposition and node labeling problem w.r.t.~a finite graph whose nodes are putative detections of body parts and w.r.t.~labels that idenfity body part classes (head, wrist, etc.) and background.
We generalize their problem to more complex objective functions.
By reducing the running time for this task compared to their branch-and-cut algorithm (that computes also lower bounds), we can tackle instances of the problem with more nodes.
This allows us to obtain more accurate pose estimates for the MPII Human Pose Dataset~\cite{andriluka14cvpr} than any published work.

Formally, the problem we propose and refer to as the minimum cost node labeling lifted multicut problem, \nllmp{}, generalizes the \textsc{np}-hard unconstrained integer quadratic program, \uiqp{}, that has been studied intensively in the context of graphical models 
\cite{kappes-2015},
and also generalizes the \textsc{np}-hard minimum cost lifted multicut problem, \lmp{} \cite{keuper-2015a}.
Unlike in pure node labeling problems such as the \uiqp{}, neighboring nodes with the same label can be assigned to distinct components, and neighboring nodes with distinct labels can be assigned to the same component.
Unlike in pure decomposition problems such as the \lmp{}, the cost of assigning nodes to the same component or distinct components can depend on node labels.
Also unlike in the \lmp{}, constraining nodes with the same label to the same component constrains the feasible decompositions to be $k$-colorable, with $k \in \mathbb{N}$ the number of labels.
For $k=2$ in particular, this constrained \nllmp{} specializes to the well-known \textsc{max-cut} problem.

In order to find feasible solutions of the \nllmp{} efficiently, we define and implement two local search algorithms that converge monotonously to a local optimum, offering a feasible solution at any time.
These algorithms do not compute lower bounds.
They output feasible solutions without approximation certificates.
Hence, they belong to the class of primal feasible heuristics for the \nllmp{}.
The first algorithm we define and refer to as alternating Kernighan-Lin search with joins and node relabeling, \alga{}, is a generalization of the algorithm KLj of Keuper et al.~\cite{keuper-2015a} and of Iterated Conditional Modes (ICM).
The second algorithm we define and refer to as joint Kernighan-Lin search with joins and node relabeling, \algb, is a generalization of KLj that transforms a decomposition and a node labeling jointly, in a novel manner.
Both algorithms build on seminal work of Kernighan and Lin \cite{kernighan-1970}.

\section{Problem}
\label{section:problem}

In this section, we define the minimum cost node labeling lifted multicut problem, \nllmp{}.
Sections~\ref{section:parameters}--\ref{section:cost-function} offer an intuition for its parameters, feasible solutions and cost function.
Section~\ref{section:definition} offers a concise and rigorous definition.
Section~\ref{section:special-cases} discusses special cases.

\subsection{Parameters}
\label{section:parameters}

Any instance of the \nllmp{} is defined with respect to the following parameters:
\begin{itemize}\parskip0ex
\item A connected graph $G = (V, E)$ whose decompositions we care about, e.g., the pixel grid graph of an image.
\item A graph $G' = (V, E')$ with $E \subseteq E'$.
This graph can contain as edges pairs of nodes that are not neighbors in $G$.
It defines the structure of the cost function.
\item A digraph $H = (V, A)$ that fixes an arbitrary orientation of the edges $E'$.
That is, for every edge $\{v,w\}$ of $G'$, the graph $H$ contains either the edge $(v,w)$ or the edge $(w,v)$.
Formally, $H$ is such that for all $v,w \in V$:
\begin{align}
& \{v,w\} \in E' \Leftrightarrow (v,w) \in A \vee (w,v) \in A \\
& (v,w) \notin A \vee (w,v) \notin A
\end{align}
\item A finite, non-empty set $L$ called the set of \emph{(node) labels}

\item The following functions whose values are called \emph{costs}:

\begin{itemize}
\item $c: V \times L \to \mathbb{R}$.
For any node $v \in V$ and any label $l \in L$, the cost $c_{vl}$ is payed iff $v$ is labeled $l$.
\item $c^\sim: A \times L^2 \to \mathbb{R}$.
For any edge $vw \in A$ and any labels $ll' \in L^2$, the cost $c^\sim_{vw,ll'}$ is payed iff $v$ is labeled $l$ and $w$ is labeled $l'$ and $v$ and $w$ are in the same component.
\item $c^{\not\sim}: A \times L^2 \to \mathbb{R}$.
For any edge $vw \in A$ and any labels $ll' \in L^2$, the cost $c^{\not\sim}_{vw,ll'}$ is payed iff $v$ is labeled $l$ and $w$ is labeled $l'$ and $v$ and $w$ are in distinct components.
\end{itemize}
\end{itemize}

\subsection{Feasible Set}
\label{section:feasible-set}

\begin{wrapfigure}{R}{0.3\linewidth}
\hspace*{-3ex}
\begin{tikzpicture}[xscale=0.6]
\node[style=vertex, label=above:$v$] at (0, 0) (v) {};
\node[style=vertex, label=above:$w$] at (2, 0) (w) {};
\draw (v) -- (w);
\node at (1, 0.2) {$e$};
\node[style=label-vertex, fill=red, label=left:{$x_{v1}$}] at (0, -0.5) {};
\node[style=label-vertex, fill=green, label=left:{$x_{v2}$}] at (0, -0.8) {};
\node[style=label-vertex, fill=blue, label=left:{$x_{v3}$}] at (0, -1.1) {};
\node[style=label-vertex, fill=red, label=right:{$x_{w1}$}] at (2, -0.5) {};
\node[style=label-vertex, fill=green, label=right:{$x_{w2}$}] at (2, -0.8) {};
\node[style=label-vertex, fill=blue, label=right:{$x_{w3}$}] at (2, -1.1) {};
\node[style=label-vertex, fill=white, draw=black] at (1, -0.65) {};
\node[style=label-vertex, fill=black, label=below:{$y_{vw}$}] at (1, -0.95) {};
\end{tikzpicture}
\end{wrapfigure}

Every feasible solution of the \nllmp{} is a pair $(x,y)$ of 01-vectors $x \in \{0,1\}^{V \times L}$ and $y \in \{0,1\}^{E'}$.
More specifically, $x$ is constrained such that, for every node $v \in V$, there is precisely one label $l \in L$ with $x_{vl} = 1$.
$y$ is constrained so as to well-define a decomposition of $G$ by the set $\{e \in E\,|\,y_e = 1\}$ of those edges that straddle distinct components.
Formally, $(x,y) \in X_{VL} \times Y_{GG'}$ with $X_{VL}$ and $Y_{GG'}$ defined below.

\begin{itemize}\parskip0ex
\item $X_{VL} \subseteq \{0,1\}^{V \times L}$, the set of all characteristic
functions of maps from $V$ to $L$, i.e., the set of all
$x \in \{0,1\}^{V \times L}$ such that
\begin{align}
    \forall v \in V: \quad
        \sum_{l \in L} x_{vl} = 1
\enspace .
\label{eq:map-constraint}
\end{align}
For any $x \in X$, any $v \in V$ and any $l \in L$ with $x_{vl} = 1$, we say that node $v$ is \emph{labeled} $l$ by $x$.
\item $Y_{GG'} \subseteq \{0,1\}^{E'}$, the set of all characteristic functions of multicuts of $G'$ lifted from $G$ \cite{andres-2015-arxiv}.
For any $y \in Y_{GG'}$ and any $e = \{v,w\} \in E'$, $y_e = 1$ indicates that $v$ and $w$ are in distinct components of the decomposition of $G$ defined by the multicut $\{e' \in E\,|\,y_{e'} = 1\}$ of $G$.
Formally,  $Y_{GG'}$ is the set of all $y \in \{0,1\}^{E'}$ that satisfy the following system of linear inequalities:
\begin{align}
& \forall C \in \textnormal{cycles}(G) \,
	\forall e \in C: \, y_e \leq \sum_{e' \in C \setminus \{e\}} y_{e'} \label{eq:cycles}
\\
& \forall \{v,w\} \in E' \setminus E \,
	\forall P \in vw\textnormal{-paths}(G): \nonumber \\
        & \qquad y_{\{v,w\}} \leq \sum_{e \in P} y_e
\\
& \forall \{v,w\} \in E' \setminus E \,
    \forall C \in vw\textnormal{-cuts}(G): \nonumber \\
        & \qquad 1 - y_{\{v,w\}} \leq \sum_{e \in C} (1 - y_e)
\enspace .
\end{align}
\end{itemize}

\subsection{Cost Function}
\label{section:cost-function}
For every $x \in \{0,1\}^{V \times L}$ and every $y \in \{0,1\}^{A \times L^2}$, a cost $\varphi(x,y) \in \mathbb{R}$ is defined by the form
\begin{align}
\varphi(x,y) = & \sum_{v \in V} \sum_{l \in L} c_{vl} \, x_{vl} \nonumber \\
& + \sum_{vw \in A} \sum_{ll' \in L^2} c^\sim_{vw,ll'} \, x_{vl} \, x_{wl'} \, (1 - y_{\{v,w\}}) \nonumber \\
& + \sum_{vw \in A} \sum_{ll' \in L^2} c^{\not\sim}_{vw,ll'} \, x_{vl} \, x_{wl'} \, y_{\{v,w\}}
\enspace .
\label{eq:objective}
\end{align}

\subsection{Definition}
\label{section:definition}

We define the \nllmp{} rigorously and concisely in the form of a linearly constrained binary cubic program.

\begin{definition}\label{definition:nllmp}
For any connected graph $G = (V, E)$,
any graph $G' = (V, E')$ with $E \subseteq E'$,
any orientation $H = (V, A)$ of $G'$,
any finite, non-empty set $L$,
any function $c: V \times L \to \mathbb{R}$ and
any functions $c^\sim, c^{\not\sim}: A \times L^2 \to \mathbb{R}$,
the instance of the \emph{minimum cost node-labeling lifted multicut problem} (\textsc{nl-lmp})
with respect to $(G, G', H, L, c, c^\sim, c^{\not\sim})$ has the form
\begin{align}
\min_{(x,y) \in X_{VL} \times Y_{GG'}} \, \varphi(x,y)
\enspace .
\label{eq:nllmp}
\end{align}
\end{definition}

\subsection{Special Cases}
\label{section:special-cases}

Below, we show that the \nllmp{} generalizes the \uiqp{}.
This connects the \nllmp{} to work on graphical models with second-order functions and finitely many labels.
In addition, we show that \nllmp{} generalizes the \lmp{}, connecting the \nllmp{} to recent work on lifted multicuts.
Finally, we show that the \nllmp{} is general enough to express subgraph selection, connectedness and disconnectedness constraints.

\subsubsection{Unconstrained Integer Quadratic Program}

\begin{definition}
	For any graph $G' = (V, E')$,
	any orientation $H = (V, A)$ of $G'$,
	any finite, non-empty set $L$,
	any $c: V \times L \to \mathbb{R}$ and
	any $c': A \times L^2 \to \mathbb{R}$,
	the instance of the \uiqp{} with respect to $(G', H, L, c, c')$ has the form
	\begin{align}
	\min_{x \in X_{VL}} \,
	\sum_{v \in V} \sum_{l \in L} c_{vl} \, x_{vl}
	+ \sum_{vw \in A} \sum_{ll' \in L^2} c'_{vw,ll'} \, x_{vl} \, x_{wl'}
	\enspace .
	\label{eq:uiqp}
	\end{align}
\end{definition}

\begin{lemma}
	For any graph $G' = (V, E')$, any instance $(G', H, L, c, c')$ of the \uiqp{} and any $x \in X_{VL}$, $x$ is a solution of this instance of the \uiqp{} iff $(x, 1_{E'})$ is a solution of the instance $(G', G', H, L, c, c', c')$ of the \nllmp{}.
\end{lemma}

\begin{proof}
	Without loss of generality, we can assume that $G'$ is connected.
	(Otherwise, we add edges between nodes $v,w \in V$ as necessary and set $c'_{vw,ll'} = 0$ for any $l,l' \in L$.)

	For any $x \in X_{GL}$, the pair $(x, 1_{E'})$ is a feasible solution of the instance of the \nllmp{} because the map $1_{E'}: E' \to \{0,1\}: e \mapsto 1$ is such that $1_{E'} \in Y_{G'G'}$.

	Moreover, $(x, 1_{E'})$ is a solution of the instance of the \nllmp{} iff $x$ is a solution of the instance of the \uiqp{} because, for $c^{\not\sim} = c^{\sim}$, the form \eqref{eq:objective} of the cost function of the \nllmp{} specializes to the form \eqref{eq:uiqp} of the cost function of the \uiqp{}.
\end{proof}

\subsubsection{Minimum Cost Lifted Multicut Problem}

\begin{definition}\cite{andres-2015-arxiv}
For any connected graph $G = (V,E)$,
any graph $G' = (V, E')$ with $E \subseteq E'$ and
any $c': E' \to \mathbb{R}$,
the instance of the minimum cost lifted multicut problem (\lmp{}) with respect to $(G, G', c')$ has the form
\begin{align}
\min_{y \in Y_{GG'}} \, \sum_{e \in E'} c'_e y_e \enspace .
\label{eq:lmp}
\end{align}
\end{definition}

\begin{lemma}
Let $(G, G', c')$ be any instance of the \lmp{}.
Let $(G, G', H, L, c, c^\sim, c^{\not\sim})$ be the instance of the \nllmp{} with the same graphs and such that
\begin{align}
L = \{1\} \quad c = 0 \quad c^\sim & = 0 \label{eq:proof-2}\\
\forall (v,w) \in A: \quad c^{\not\sim}_{vw,11} & = c'_{\{v,w\}} \enspace . \label{eq:proof-3}
\end{align}
Then, for any $y \in \{0,1\}^{E'}$, $y$ is a solution of the instance of the \lmp{} iff $(1_{V \times L},y)$ is a solution of the instance of the \nllmp{}.
\end{lemma}

\begin{proof}
Trivially, $y$ is a feasible solution of the instance of the \lmp{} iff $(1_{V \times L}, y)$ is a feasible solution of the instance of the \nllmp{}.
More specifically, $y$ is a solution of the instance of the \lmp{} iff $(1_{V \times L}, y)$ is a solution of the instance of the \nllmp{} because, for any $x \in X_{VL}$, the cost function \eqref{eq:objective} of the \nllmp{} assumes the special form below which is identical with the form in \eqref{eq:lmp}.
\begin{align}
\varphi(x,y)
\overset{\eqref{eq:map-constraint},\eqref{eq:proof-2}}{=} \sum_{vw \in A} c^{\not\sim}_{vw,11} y_{\{v,w\}}
\overset{\eqref{eq:proof-3}}{=} \sum_{e \in E'} c'_e y_e
\enspace .
\end{align}
\end{proof}

\subsubsection{Subgraph Selection}
\label{sec:suppression}

Applications such as~\cite{insafutdinov-2016,pishchulin-2016,tang-2015,tang-2016} require us to not only decompose a graph and label its nodes but to also select a subgraph.
The \nllmp{} is general enough to model subgraph selection.
To achieve this, one proceeds in two steps:
Firstly, one introduces a special label $\epsilon \in L$ to indicate that a node is not an element of the subgraph.
We call these nodes \emph{inactive}.
All other nodes are called \emph{active}.
Secondly, one chooses a large enough $c^* \in \mathbb{N}$, a $c^\dagger \in \mathbb{N}_0$ and $c^\sim, c^{\not\sim}$ such that
\begin{align}
\forall vw \in A \, \forall l \in L \setminus \{\epsilon\}: \quad
    & c^\sim_{vw, l \epsilon} = c^\sim_{vw, \epsilon l} = c^* \label{eq:subgraph-1}\\
    & c^{\not\sim}_{vw, l \epsilon} = c^{\not\sim}_{vw, \epsilon l} = 0 \label{eq:subgraph-2}\\
\forall vw \in A: \quad
    & c^\sim_{vw, \epsilon \epsilon} = c^\dagger \label{eq:subgraph-3}
\enspace .
\end{align}

By \eqref{eq:subgraph-1}, inactive nodes are not joined with active nodes in the same component.
By \eqref{eq:subgraph-2}, cutting an inactive node from an active node has zero cost.
By \eqref{eq:subgraph-3}, joining inactive nodes has cost $c^\dagger$, possibly zero.
Choosing $c^\dagger$ large enough implements an additional constraint proposed in \cite{tang-2015} that inactive nodes are necessarily isolated.
It is by this constraint and by a two-elementary label set that \cite{tang-2015} is a specialization of the \nllmp{}.

\subsubsection{(Dis-)Connectedness Constraints}
\label{section:connectedness}

Some applications require us to constrain certain nodes to be in distinct components.
One example is instance-separating semantic segmentation where nodes with distinct labels necessarily belong to distinct segments \cite{kroeger-2014}.
Other applications require us to constrain certain nodes to be in the same component.
One example is articulated human body pose estimation for a single human in the optimization framework of \cite{pishchulin-2016} where every pair of active nodes necessarily belongs to the same human.
Another example is connected foreground segmentation \cite{nowozin-2010,rempfler-2016,stuehmer-2015,vicente-2008} in which every pair of distinct foreground pixels necessarily belongs to the same segment.

The \nllmp{} is general enough to model a combination of connectedness constraints and disconnectedness constraints by sufficiently large costs:
In order to constrain distinct nodes $v,w \in V$ with labels $l,l' \in L$ to be in \emph{the same component}, one introduces an edge $(v,w) \in A$, a large enough $c^* \in \mathbb{N}$ and costs $c^\nsim$ such that $c^\nsim_{vw, ll'} = c^\nsim_{vw, l'l} = c^*$.
In order to constrain distinct nodes $v,w \in V$ with labels $l,l' \in L$ to be in \emph{distinct components}, one introduces an edge $(v,w) \in A$, a large enough $c^* \in \mathbb{N}$ and costs $c^{\sim}$ such that $c^{\sim}_{vw, ll'} = c^{\sim}_{vw, l'l} = c^*$.

\section{Algorithms}
\label{section:algorithm}

In this section, we define two local search algorithms that compute feasible solutions of the \nllmp{} efficiently.
Both algorithms attempt to improve a current feasible solution recursively by \emph{transformations}.
One class of transformations alters the node labeling of the graph by replacing a single node label.
A second class of transformations alters the decomposition of the graph by moving a single node from one component to another.
A third class of transformations alters the decomposition of the graph by joining two components.

As proposed by Kernighan and Lin~\cite{kernighan-1970} and generalized to the \lmp{} by Keuper et al.~\cite{keuper-2015a}, a local search is carried out not over the set of individual transformations of the current feasible solution but over a set of sequences of transformations.
Complementary to this idea, we define and implement two schemes of combining transformations of the decomposition of the graph with transformations of the node labeling of the graph.
This leads us to define two local search algorithms for the \nllmp{}.

\subsection{Encoding Feasible Solutions}
To encode feasible solutions $(x,y) \in X_{VL} \times Y_{GG'}$ of the \nllmp{}, we consider two maps:
A \emph{node labeling} $\lambda: V \to L$ that defines the $x^\lambda \in X_{VL}$ such that
\begin{align}
\forall v \in V \,
\forall l \in L: \quad
    x^\lambda_{vl} = 1 \,\Leftrightarrow\, \lambda(v) = l
\enspace ,
\end{align}
and a so-called \emph{component labeling} $\mu: V \to \mathbb{N}$ that defines the $y^\mu \in \{0,1\}^{E'}$ such that
\begin{align}
\forall \{v,w\} \in E': \quad
    y^\mu_{\{v,w\}} = 0 \,\Leftrightarrow\, \mu(v) = \mu(w)
\enspace .
\end{align}

\subsection{Transforming Feasible Solutions}
To improve feasible solutions of the \nllmp{} recursively, we consider three transformations of the encodings $\lambda$ and $\mu$:

For any node $v \in V$ and any label $l \in L$, the transformation $T_{vl}: L^V \to L^V: \lambda \mapsto \lambda'$ changes the label of the node $v$ to $l$, i.e.
\begin{align}
\forall w \in V: \quad
\lambda'(w) := \begin{cases}
    l   & \textnormal{if}\ w = v \\
    \lambda(w) & \textnormal{otherwise}
\end{cases}
\enspace .
\label{eq:transform-relabel}
\end{align}

For any node $v \in V$ and any component index $m \in \mathbb{N}$, the transformation $T'_{vm}: \mathbb{N}^V \to \mathbb{N}^V: \mu \mapsto \mu'$ changes the component index of the node $v$ to $m$, i.e.
\begin{align}
\forall w \in V: \quad
\mu'(w) := \begin{cases}
    m   & \textnormal{if}\ w = v \\
    \mu(w) & \textnormal{otherwise}
\end{cases}
\enspace .
\label{eq:transform-cut}
\end{align}

For any component indices $m,m' \in \mathbb{N}$, the transformation $T'_{mm'}: \mathbb{N}^V \to \mathbb{N}^V: \mu \mapsto \mu'$ puts all nodes currently in the component indexed by $m$ into the component indexed by $m'$, i.e.
\begin{align}
	\forall w \in V: \quad
	\mu'(w) := \begin{cases}
		m'   & \textnormal{if}\ \mu(w) = m\\
		\mu(w) & \textnormal{otherwise}
	\end{cases}
	\enspace .
\label{eq:transform-join}
\end{align}

Not every component labeling $\mu$ is such that $y^{\mu} \in Y_{GG'}$.
In fact, $y^{\mu}$ is feasible if and only if, for every $m \in \mu(V)$, the node set $\mu^{-1}(m)$ is connected in $G$.
For efficiency, we allow for transformations \eqref{eq:transform-cut} whose output $\mu'$ violates this condition, as proposed in \cite{keuper-2015a}.
This happens when an \emph{articulation node} of a component is moved to a different component.
In order to \emph{repair} any $\mu'$ for which $y^{\mu}$ is infeasible, we consider a map $R: \mathbb{N}^V \to \mathbb{N}^V: \mu' \mapsto \mu$ such that, for any $\mu': V \to \mathbb{N}$ and any distinct $v,w \in V$, we have $\mu(v) = \mu(w)$ if and only if the exists a $vw$-path in $G$ along which all nodes have the label $\mu'(v)$.
We implement $R$ as connected component labeling by breadth-first-search.

\subsection{Searching Feasible Solutions}

We now define two local search algorithms that attempt to improve an initial feasible solution recursively, by applying the transformation defined above.
Initial feasible solutions are given, for instance, by the finest decomposition of the graph $G$ that puts every node in a distinct component, or by the coarsest decomposition of the graph $G$ that puts every node in the same component, each together with any node labeling.
We find an initial feasible solution for our local search algorithm by first fixing an optimal label for every node independently and by then solving the resulting \lmp{}, i.e., \eqref{eq:nllmp} for the fixed labels $x \in X_{VL}$, by means of greedy agglomerative edge contraction \cite{keuper-2015a}.

\textbf{\alga{} Algorithm.}
The first local search algorithm we define, alternating Kernighan-Lin search with joins and node relabeling, \alga{}, alternates between transformations of the node labeling and transformations of the decomposition.
For a fixed decomposition, the labeling is transformed by Func.~\ref{function:update-labeling} which greedily updates labels of nodes independently.
For a fixed labeling, the decomposition is transformed by Func.~\ref{function:update-multicut}, \emph{without those parts of the function that are written in green}, i.e., precisely the algorithm KLj of \cite{keuper-2015a}.
(All symbols that appear in the pseudo-code are defined above, except the iteration counter $t$, cost differences $\delta, \Delta$, and 01-vectors $\alpha$ used for bookkeeping, to avoid redundant operations.)

\textbf{\algb{} Algorithm.}
The second local search algorithm we define, joint Kernighan-Lin search with joins and node relabeling, \algb{}, transforms the decomposition and the node labeling jointly, by combining the transformations \eqref{eq:transform-relabel}--\eqref{eq:transform-join} in a novel manner.
It is given by Func.~\ref{function:update-multicut}, \emph{with those parts of the function that are written in green}.

Like the alternating algorithm \alga{}, the joint algorithm \algb{} updates the labeling for a fixed decomposition (calls of Func.~\ref{function:update-labeling} from Func.~\ref{function:update-multicut}).
Unlike the alternating algorithm \alga{}, the joint algorithm \algb{} updates the decomposition and the labeling also jointly.
This happens in Func.~\ref{function:update-two-cut} that is called from \algb{}, \emph{with the part  that is written in green}.

\begin{function}[t]
\caption{$(\Delta, \lambda') =\,$update-labeling$(\mu, \lambda)$}
\label{function:update-labeling}
\vspace{-3ex}
\flushleft
$\lambda_0 := \lambda \quad$ 
$\Delta := 0 \quad$
$t := 0$ \\
\kw{repeat} \\
\ind \kw{choose} $(\hat v, \hat l) \in \!\! \underset{(v, l) \in V \times L}{\mathrm{argmin}} \varphi(x^{T_{vl}(\lambda_t)}, y^{\mu_t}) - \varphi(x^{\lambda_t}, y^{\mu_t})$ \\
\ind $\delta := \varphi(x^{T_{\hat v \hat l}(\lambda_t)}, y^{\mu_t}) - \varphi(x^{\lambda_t}, y^{\mu_t})$ \\
\ind \kw{if} $\delta < 0$ \\
\ind \ind $\lambda_{t+1} := T_{\hat v \hat l}(\lambda_t)$ \\
\ind \ind $\Delta := \Delta + \delta$ \\
\ind \ind $t := t + 1$ \\
\ind \kw{else} \\
\ind \ind \kw{return} $(\Delta, \lambda_t)$
\end{function}

Func.~\ref{function:update-two-cut} looks at two components $V := \mu^{-1}(m)$ and $W := \mu^{-1}('m)$ of the current decomposition.
It attempts to improve the decomposition as well as the labeling by moving a node from $V$ to $W$ or from $W$ to $V$ \emph{and by simultaneously changing its label}.
As proposed by Kernighan and Lin \cite{kernighan-1970}, Func.~\ref{function:update-two-cut} does not make such transformations greedily but first constructs a sequence of such transformations greedily and then executes the first $k$ with $k$ chosen so as to decrease the objective value maximally.
\alga{} constructs a sequence of transformations analogously, but the node labeling remains fixed throughout every transformation of the decomposition.
Thus, \algb{} is a local search algorithm whose local neighborhood is strictly larger than that of \alga{}.

\begin{function}[b]
\caption{$(\Delta', \mu' \hw{, \lambda'}) =\,$update-lifted-multicut$(\mu\hw{, \lambda})$}
\label{function:update-multicut}
\vspace{-3ex}
\flushleft
$\mu_0 := \mu \qquad t : = 0$ \\
\hw{$(\delta, \lambda_0) := \textnormal{update-labeling}(\mu_0, \lambda)$} \\
\kw{let} $\alpha_0: \mathbb{N} \to \{0,1\}$ such that $\alpha_0(\mathbb{N}) = 1$ \\
\kw{repeat} \\
\ind $\Delta := 0 \quad$ 
\ind $\mu_{t+1} := \mu_t \quad$ 
\ind \hw{$\lambda_{t+1} := \lambda_t$} \\
\ind \kw{let} $\alpha_{t+1}: \mathbb{N} \to \{0,1\}$ such that $\alpha_{t+1}(\mathbb{N}) = 0$ \\
\ind \kw{for each} $\{m, m'\} \in \tbinom{\mu(V)}{2}$ \\
\ind \ind \kw{if} $\alpha_t(m) = 0 \wedge \alpha_t(m') = 0$ \\
\ind \ind \ind \kw{continue} \\
\ind \ind $(\delta, \mu_{t+1}\hw{, \lambda_{t+1}}) := \,$update-2-cut$(\mu_{t+1}\hw{, \lambda_{t+1}}, m, m')$ \\
\ind \ind \kw{if} $\delta < 0$ \\
\ind \ind \ind $\alpha_{t+1}(m) := 1 \quad \alpha_{t+1}(m') := 1 \quad \Delta := \Delta + \delta$ \\
\ind \kw{for each} $m \in \mu(V)$\\
\ind \ind \kw{if} $\alpha_t(m) = 0$ \\
\ind \ind \ind \kw{continue} \\
\ind \ind $m' := 1 + \max \mu(V)$ \hfill (new component)\\
\ind \ind $(\delta, \mu_{t+1}\hw{, \lambda_{t+1}}) := \,$update-2-cut$(\mu_{t+1}\hw{, \lambda_{t+1}}, m, m')$ \\
\ind \ind \kw{if} $\delta < 0$ \\
\ind \ind \ind $\alpha_{t+1}(m) := 1 \quad \alpha_{t+1}(m') := 1 \quad \Delta := \Delta + \delta$ \\
\ind \hw{$(\delta, \lambda_{t+1}) := \textnormal{update-labeling}(\mu_{t+1}, \lambda_{t+1})$} \\
\ind \hw{$\Delta := \Delta + \delta$} \\
\ind \kw{if} $y^{\mu_{t+1}} \notin Y_{GG'}$ \\
\ind \ind $\mu_{t+1} := R(\mu_{t+1})$ \hfill (repair heuristic)\\
\ind \ind $\Delta := \varphi(x^{\hw{\lambda_{t+1}}}, y^{\mu_{t+1}}) - \varphi(x^{\hw{\lambda_0}}, y^{\mu_0})$ \\
\ind $t := t + 1$ \\
\kw{while} $\Delta < 0$
\end{function}

Our \textsc{c++} implementation computes cost differences incrementally and solves the optimization problem over transformations by means of a priority queue, as described in detail in Appendix~\ref{app:alg_details}.
The time and space complexities are identical to those of KLj and are established in \cite{keuper-2015a}, as transformations that take linear time in the number of labels take constant time in the size of the graph.

\begin{function}
\caption{$(\Delta', \mu' \hw{, \lambda'}) =\,$update-2-cut$(\mu\hw{, \lambda}, m, m')$}
\label{function:update-two-cut}
\vspace{-3ex}
\flushleft
$\mu_0 := \mu \qquad$ \hw{$\lambda_0 := \lambda$} $\qquad t := 0$\\
\kw{if} $\mu^{-1}(m') = \emptyset$ \\
\ind $V_0 := \mu^{-1}(m)$ \\
\kw{else} \\
\ind $V_0 := \{v \in \mu^{-1}(m) \,|\, \exists w \in \mu^{-1}(m'): \, \{v,w\} \in E \}$ \\
\kw{if} $\mu^{-1}(m) = \emptyset$ \\
\ind $W_0 := \mu^{-1}(m')$ \\
\kw{else} \\
\ind $W_0 := \{w \in \mu^{-1}(m') \,|\, \exists v \in \mu^{-1}(m): \, \{v,w\} \in E \}$ \\
\kw{let} $\alpha: \mathbb{N} \to \{0,1\}$ such that $\alpha(\mathbb{N}) = 1$ \\
\kw{while} $V_t \cup W_t \neq \emptyset$ \\
\ind $\delta := \delta' := \infty$ \\
\ind \kw{if} $V_t \neq \emptyset$ \\
\ind \ind \kw{choose} $\hw{(} \hat v \hw{, \hat l)} \in \underset{\hw{(} v \hw{, l)} \in V_t \hw{\times L}}{\mathrm{argmin}} \varphi(x^{\hw{T_{vl}(\lambda_t)}}, y^{T'_{v m'}(\mu_t)}) \, -$ \\[-2ex]
	\hfill $\varphi(x^{\hw{\lambda_t}}, y^{\mu_t})$ \\[1ex]
\ind \ind $\delta := \varphi(x^{\hw{T_{\hat v \hat l}(\lambda_t)}}, y^{T'_{\hat v m'}(\mu_t)}) - \varphi(x^{\hw{\lambda_t}}, y^{\mu_t})$ \\
\ind \kw{if} $W_t \neq \emptyset$ \\
\ind \ind \kw{choose} $\hw{(} \hat w \hw{, \hat l)} \in \underset{\hw{(} w \hw{, l)} \in W_t \hw{\times L}}{\mathrm{argmin}} \varphi(x^{\hw{T_{wl}(\lambda_t)}}, y^{T'_{w m}(\mu_t)}) \, -$ \\[-2ex]
	\hfill $\varphi(x^{\hw{\lambda_t}}, y^{\mu_t})$ \\[1ex]
\ind \ind $\delta' := \varphi(x^{\hw{T_{\hat w \hat l}(\lambda_t)}}, y^{T'_{\hat w m}(\mu_t)}) - \varphi(x^{\hw{\lambda_t}}, y^{\mu_t})$ \\
\ind \kw{if} $\delta \leq \delta'$ \\
\ind \ind $\mu_{t+1} := T'_{\hat v m'}(\mu_t)$ \hfill (move node $\hat v$ to component $m'$)\\
\ind \ind \hw{$\lambda_{t+1} := T_{\hat v \hat l}(\lambda_t)$ \hfill (label node $\hat v$ with label $\hat \lambda$)}\\
\ind \ind $\alpha(\hat v) := 0$ \hfill (mark $\hat v$ as inactive) \\
\ind \kw{else} \\
\ind \ind $\mu_{t+1} := T'_{\hat w m}(\mu_t)$ \hfill (move node $\hat w$ to component $m$)\\
\ind \ind \hw{$\lambda_{t+1} := T_{\hat w \hat l}(\lambda_t)$ \hfill (label node $\hat w$ with label $\hat \lambda$)}\\
\ind \ind $\alpha(\hat w) := 0$ \hfill (mark $\hat w$ as inactive) \\
\ind $V_{t+1} := \{v \in V \,|\, \mu_{t+1}(v) = m \wedge \alpha(v) = 1 \,\wedge$\\
	\hfill $\exists \{v,w\} \in E: \, \mu_{t+1}(w) = m' \}$ \\
\ind $W_{t+1} := \{w \in V \,|\, \mu_{t+1}(w) = m' \wedge \alpha(w) = 1 \,\wedge$\\
	\hfill $\exists \{v,w\} \in E: \, \mu_{t+1}(v) = m \}$ \\
\ind $t := t + 1$ \\
$\hat t := \min \underset{t' \in \{0,\ldots,t\}}{\mathrm{argmin}} \varphi(x^{\hw{\lambda_{t'}}}, y^{\mu_{t'}}) - \varphi(x^{\hw{\lambda_0}}, y^{\mu_0})$ \\
$\Delta_1 := \varphi(x^{\hw{\lambda_{\hat t}}}, y^{\mu_{\hat t}}) - \varphi(x^{\hw{\lambda_0}}, y^{\mu_0})$ \\
$\Delta_2 := \varphi(x^{\hw{\lambda_0}}, y^{T'_{mm'}(\mu)}) - \varphi(x^{\hw{\lambda_0}}, y^{\mu_0})$ \hfill (join $m$ and $m'$)\\
\kw{if} $\min \{\Delta_1, \Delta_2\} \geq 0$ \\
\ind \kw{return} $(0, \mu\hw{, \lambda})$ \\
\kw{else if} $\Delta_1 < \Delta_2$ \\
\ind \kw{return} $(\Delta_1, \mu_{\hat t} \hw{, \lambda_{\hat t}})$ \\
\kw{else} \\
\ind \kw{return} $(\Delta_2, T_{mm'}(\mu) \hw{, \lambda})$
\end{function}

\section{Applications}
\label{section:applications}

We show applications of the proposed problem and algorithms to three distinct computer vision tasks: articulated human body pose estimation, multiple object tracking, and instance-separating semantic segmentation.
For each task, we set up instances of the \nllmp{} from published data, using published algorithms.

\subsection{Articulated Human Body Pose Estimation}
\label{section:pose}

We turn toward applications of the \nllmp{} and the algorithms \alga{} and \algb{} to the task of estimating the articulated poses of all humans visible in an image. 
Pishchulin et al.~\cite{pishchulin-2016} and Insafutdinov et al.~\cite{insafutdinov-2016} approach this problem via a graph decomposition and node labeling problem that we identify as a special case of the \nllmp{} with $c^{\not\sim} = 0$ and with subgraph selection (Section~\ref{sec:suppression}).
We relate their notation to ours in Appendix~\ref{app:pose}.
Nodes in their graph are putative detections of body parts.
Labels define body part classes (head, wrist, etc.).
In our notation, $x_{vl} = 1$ indicates that the putative detection $v$ is a body part of class $l$, and $y_{vw} = 1$ indicates that the body parts $v$ and $w$ belong to distinct humans.
The test set of \cite{insafutdinov-2016} consists of 1758 such instances of the \nllmp{}.

To tackle these instances, Insafutdinov et al.~define and implement a branch-and-cut algorithm in the integer linear programming software framework Gurobi.
We refer to their published \textsc{c++} implementation as \textsc{b\&c}.

\textbf{Cost and time.}
In Fig.~\ref{fig:pose_time_data}, we compare the convergence of \textsc{b\&c} (feasible solutions and lower bounds) with the convergence of our algorithms, \alga{} and \algb{} (feasible solutions only).
Shown in this figure is the average objective value over the test set w.r.t.~the absolute running time.
Thanks to the lower bounds obtained by \textsc{b\&c}, it can be seen from this figure that \alga{} and KL+r arrive at near optimal feasible solutions after $10^{-1}$ seconds, five orders of magnitude faster than \textsc{b\&c}.
This result shows that primal feasible heuristics for the \nllmp{}, such as \alga{} and \algb{}, are practically useful in the context of this application.

\textbf{Application-specific accuracy.}
In Tab.~\ref{tab:pose_results}, we compare feasible solutions output by \alga{} and \algb{} after convergence with those obtained by \textsc{b\&c} after at most three hours.
It can be seen from this table that the feasible solutions output by \alga{} and \algb{} have lower cost and higher application-specific accuracy (Acc) on average. 
\algb{} yields a lower average cost than \alga{} with slightly higher running time. 
The fact that lower cost does not mean higher application-specific accuracy is explained by the application-specific accuracy measure that does not penalize false positives.

The shorter absolute running time of \alga{} and \algb{} allows us to increase the number of nodes from 150, as in \cite{insafutdinov-2016}, to 420. 
It can be seen from the last two rows of Tab.~\ref{tab:pose_results} that this increases the application-specific accuracy by 4\%.

\begin{figure}
\centering
\begin{tikzpicture}[font=\small]
\begin{axis}[
		scaled ticks=base 10:-4,
        width=0.9\columnwidth,
        height=0.35\textwidth,
        xmode=log,
        xlabel={Running time [s]},
        ylabel={Objective Value},
        ylabel near ticks,
        xlabel near ticks,
        mark size=0.1ex,
        legend entries={\alga{}, \algb{}, \textsc{b\&c}, \textsc{b\&c} bound},
        legend style={draw=none, fill=none, nodes={right}, at={(0.95,0.9)},anchor=north east}
    ]

    \addplot[
            mark=*,
            color=r
        ] table [
            col sep=comma,
            x = RT, y = E
        ] {time_data/simple.txt};

    \addplot[
            mark=*,
            color=brown
        ] table [
            col sep=comma,
            x = RT, y = E
        ] {time_data/joint.txt};

    \addplot[
            mark=*,
            color=g
        ] table [
            col sep=comma,
            x = RT, y = E
        ] {time_data/ilp.txt};

    \addplot[
            mark=*,
            color=black
        ] table [
            col sep=comma,
            x = RT, y = LB
        ] {time_data/ilp.txt};
\end{axis}
\end{tikzpicture}
\caption{Convergence of \textsc{b\&c}, \alga{} and \algb{} in an application to the task of articulated human body pose estimation.}
\label{fig:pose_time_data}
\end{figure}
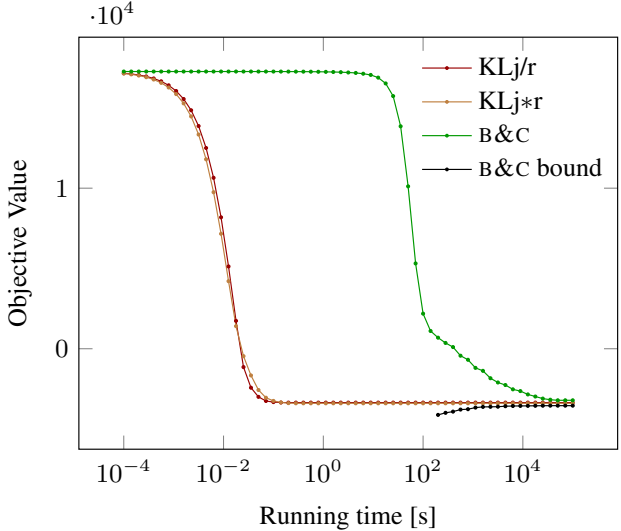

\addtolength{\tabcolsep}{-2pt}
\begin{table}
\centering
\small
\begin{tabular}{@{}clrrrr@{}}
\toprule
$|V|$ & Alg. & AP & Mean cost & Mean time [s] & Median time [s] \\
\midrule
& \cite{insafutdinov-2016} & 65.5 & -3013.30 & 9519.26 & 308.28 \\
\cmidrule{2-6}
& \alga{} & 66.5 & -3352.74 & {\bf0.033} & {\bf0.031} \\
\multirow{-3}{*}{\begin{sideways}150\end{sideways}} & \algb{} & \bf 66.6 & \bf -3419.07 & 0.119 & 0.100 \\
\midrule
& \alga{} & \bf 70.6 & -6184.36 & \bf 0.098 & \bf 0.053 \\
\multirow{-2}{*}{\begin{sideways}420\end{sideways}} & \algb{} & \bf 70.6 & \bf -6608.53 & 0.534 & 0.254 \\
\bottomrule
\end{tabular}\\[-2ex]
\caption{Comparison of \textsc{b\&c}~\cite{insafutdinov-2016}, \alga{} and \algb{} in an application to the task of human body pose estimation.}
\label{tab:pose_results}
\end{table}
\addtolength{\tabcolsep}{2pt}

\subsection{Instance-Separating Semantic Segmentation}

We turn toward applications of the \nllmp{} and the algorithms \alga{} and \algb{} to the task of instance-separating semantic image segmentation.
We state this problem here as an \nllmp{} whose nodes correspond to pixels in a given image, and whose labels define classes of objects (human, car, bicycle, etc.).
In our notation, $x_{vl} = 1$ indicates that the pixel $v$ shows an object of class $l$, and $y_{vw} = 1$ indicates that the pixels $v$ and $w$ belong to distinct objects.

Specifically, we apply the algorithms \alga{} and \algb{} to instances of the \nllmp{} for the task of instance-separating semantic segmentation posed by the KITTI \cite{kitti12cvpr} and Cityscapes \cite{cityscapes16cvpr} benchmarks.
For KITTI, we construct instances of the \nllmp{} from data published by Uhrig et al.~\cite{uhrig16gcpr} as described in detail in Appendix~\ref{app:instance}.
For Cityscapes, we construct instances of the \nllmp{} as follows. 
For costs $c^\nsim$, we again use data of Uhrig et al.~\cite{uhrig16gcpr}.
For costs $c$, we use a ResNet-50~\cite{He_2016_CVPR} network with dilated convolutions~\cite{CP2016Deeplab}. 
We train the network in a fully convolutional manner with image crops (768\,px$\cdot$512\,px) subjected to minimal data augmentation (horizontal flips).
More details are in Appendix~\ref{app:instance}.

\begin{figure}
\centering
\begin{tikzpicture}[font=\small]
\begin{axis}[
		scaled ticks=base 10:-6,
        width=0.8\columnwidth,
        height=0.2\textwidth,
        xmode=log,
        ylabel={Objective Value},
        ylabel near ticks,
        xlabel near ticks,
        mark size=0.1ex,
        legend entries={KITTI \alga{},KITTI \algb{}},
        legend style={draw=none, fill=none, nodes={right}, at={(0.95,0.9)},anchor=north east} 
    ]

    \addplot[
            mark=*,
            color=r
        ] table [
            col sep=comma,
            x = RT, y = E
        ] {time_data/kitti-simple.txt};

    \addplot[
            mark=*,
            color=brown
        ] table [
            col sep=comma,
            x = RT, y = E
        ] {time_data/kitti-joint.txt};
\end{axis}
\end{tikzpicture}
\begin{tikzpicture}[font=\small]
\begin{axis}[
		scaled ticks=base 10:-6,
        width=0.8\columnwidth,
        height=0.2\textwidth,
        xmode=log,
        xlabel={Time [s]},
        ylabel={Objective Value},
        ylabel near ticks,
        xlabel near ticks,
        mark size=0.1ex,
        legend entries={CS \alga{},CS \algb{}},
        legend style={draw=none, fill=none, nodes={right}, at={(0.95,0.9)},anchor=north east} 
    ]

    \addplot[
            mark=*,
            color=r
        ] table [
            col sep=comma,
            x = RT, y = E
        ] {time_data/cs-simple.txt};

    \addplot[
            mark=o,
            color=brown
        ] table [
            col sep=comma,
            x = RT, y = E
        ] {time_data/cs-joint.txt};
\end{axis}
\end{tikzpicture}
\caption{Convergence of \alga{} and \algb{} in an application to the task of instance-separating semantic segmentation.}
\label{fig:kitti_time_data}
\end{figure}
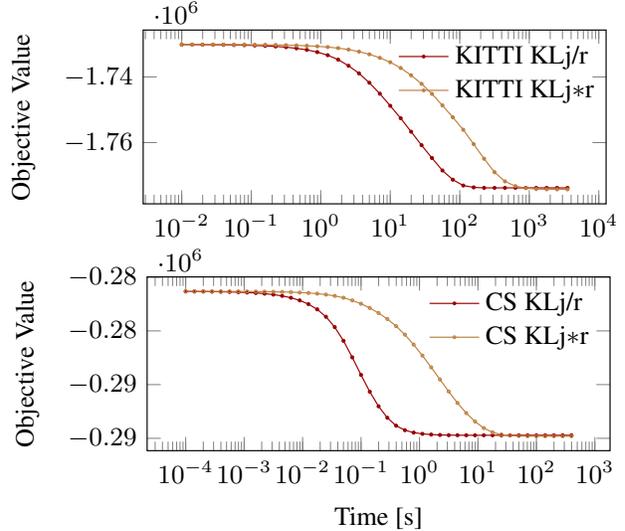

\begin{table}
\setlength{\tabcolsep}{4pt}
\renewcommand{\arraystretch}{1}
\centering
\small
\begin{tabular}{llrr}
\toprule
Data & Algorithm & $\mapr$ & $\mapr^{50\%}$    \\
\midrule
KITTI validation \cite{kitti12cvpr}   & \alga{}     & $\bf 50.5$    & $\bf 82.9$        \\
                                      & \algb{}     & $50.3$        & $82.4$            \\
\midrule
KITTI test \cite{kitti12cvpr} & \cite{uhrig16gcpr} & $41.6$        & $69.1$             \\
                              & \algb{}            & $\bf 43.6$    & $\bf 71.4$         \\
\midrule
Cityscapes validation \cite{cityscapes16cvpr}   & \alga{}     & $11.3$        & $\bf 26.8$    \\
                                                & \algb{}     & $\bf 11.4$    & $26.1$        \\
\midrule
Cityscapes test \cite{cityscapes16cvpr}  & MCG+R-CNN \cite{cityscapes16cvpr} & $4.6$         & $12.9$             \\
                                         & \cite{uhrig16gcpr}                & $8.9$         & $21.1$             \\
                                         & \algb{}                           & $\bf 9.8$     & $\bf 23.2$         \\
\bottomrule
\end{tabular}\\[-2ex]
\caption{Comparison of \alga{} and \algb{} in an application to the task of instance-separating semantic segmentation.}
\label{tab:instance_results_CS}
\end{table}

\textbf{Cost and time.}
In Fig.~\ref{fig:kitti_time_data}, we compare the convergence of \alga{} and \algb{}.
Shown in this figure w.r.t.~the absolute running time are the average objective values over the KITTI and Cityscapes validation sets, respectively.
It can be seen from this figure that \alga{} converges faster than \algb{}. 
Both algorithms are practical for this application but not efficient enough for video processing in real-time.

\textbf{Application-specific accuracy.}
In Tab.~\ref{tab:instance_results_CS}, we compare feasible solutions output by \alga{} and \algb{} after convergence with the output of the algorithm of Uhrig et al~\cite{uhrig16gcpr}.
It can be seen from this table that the application of \alga{} and \algb{} improves the application-specific average precision, $\mapr$ and $\mapr^{50\%}$.
The AP of feasible solutions output by \algb{} for the Cityscapes test set is higher than that of any published algorithm.
A higher AP is reported by Kirillov et al.~\cite{kirillov-2017}, who use the model and algorithms proposed in this paper with improved pairwise $c^\nsim$ and unary $c$ costs.

\subsection{Multiple Object Tracking}
\label{section:tracking}

\begin{table*}[t!!!]
\centering
\small
\begin{tabular}{l r r r  r  r r r  r r r r }
      \toprule
      Method                          & MOTA $\uparrow$ & MOTP $\uparrow$ & FAF $\downarrow$ & MT $\uparrow$ & ML $\downarrow$& FP $\downarrow$ & FN $\downarrow$& ID Sw $\downarrow$& Frag$\downarrow$ & Hz $\uparrow$ & Detector \\
      \midrule
      \cite{FagotBouquet2016}     & 41.0 & 74.8 &  1.3 & 11.6\%      & 51.3\% &7896       & 99224  &430   &963  & 1.1 & Public\\
      \cite{kim_ICCV2015_MHTR} & 42.9 & 76.6 & 1.0 & 13.6\%      & 46.9\% & \bf{5668} & 97919 & 499  & 659 & 0.8 & Public    \\
      \cite{Choi15}                        & 46.4 & 76.6 & 1.6  & \bf{18.3}\% & 41.4\% & 9753 & \bf{87565} & \bf{359}   & \bf{504}  & 2.6 & Public    \\
      \midrule
      \cite{tang-2016}                    & 46.3 & 75.7 & 1.1 & 15.5\% & \bf{39.7}\% & 6373 & 90914 & 657   & 1114  & 0.8 & Public    \\
      \alga{}                                   & \bf{47.6} & \bf{78.5} & 1.0 & 17.0\% & 40.4\% & 5844 & 89093 &629 &768 & \bf{8.3}  & Public \\
      \algb{}                                   & \bf{47.6} & \bf{78.5} & \bf{0.98} & 17.0\% & 40.4\% & 5783 &89160 &627 &761 & 0.7  & Public \\
      \bottomrule
    \end{tabular}
    \caption{Comparison of the algorithms \alga{} and \algb{} in an application to the task of multiple object tracking.\vspace{-9pt}}
    \label{tab:mot-result}
\end{table*}

\begin{figure}
\centering
\begin{tikzpicture}[font=\small]
\begin{axis}[
		scaled ticks=base 10:-6,
        width=0.8\columnwidth,
        height=0.2\textwidth,
        xmode=log,
        xlabel={Running time [s]},
        ylabel={Objective Value},
        ylabel near ticks,
        xlabel near ticks,
        mark size=0.1ex,
        legend entries={\alga{}, \algb{}},
        legend style={draw=none, fill=none, nodes={right}, at={(0.95,0.9)}, anchor=north east} 
    ]

    \addplot[
            mark=*,
            color=r,
        ] table [
            col sep=comma,
            x = RT, y = E
        ] {time_data/mot-simple.txt};

    \addplot[
            mark=*,
            color=brown
        ] table [
            col sep=comma,
            x = RT, y = E
        ] {time_data/mot-joint.txt};
\end{axis}
\end{tikzpicture}
\caption{Convergence of the algorithms \alga{} and \algb{} in an application to the task of multiple object tracking.}
\label{fig:mot_time_data}
\end{figure}
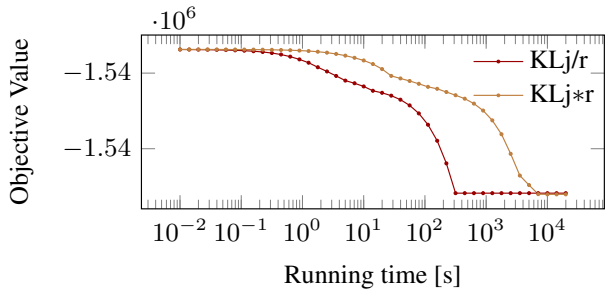

We turn toward applications of the \nllmp{} and the algorithms \alga{} and \algb{} to the task of multiple object tracking. 
Tang et al.~\cite{tang-2015} approach this problem via a graph decomposition and node labeling problem that we identify as a special case of the \nllmp{} with two labels and subgraph selection (Sec.~\ref{sec:suppression}).
We relate their notation to ours rigorously in Appendix~\ref{app:mot}.
Nodes in their graph are putative detections of persons.
In our notation, $x_{vl} = 1$ indicates that the putative detection $v$ is active, and $y_{vw} = 1$ indicates that the putative detections $v$ and $w$ are of distinct persons.
For the test set of the multiple object tracking benchmark~\cite{MilanL0RS16}, Tang et al.~construct seven such instances of the \nllmp{}.

To tackle these large instances, in~\cite{tang-2016} Tang et al. solve the subgraph suppression problem first and independently, by thresholding on the detections scores, and then solve the minimum cost multicut problem for the remaining subgraph by means of the algorithm KLj of \cite{keuper-2015a}, without re-iterating.
Here, we apply to the joint \nllmp{} the algorithms \alga{} and \algb{} and compare their output to that of~\cite{tang-2016} and of other top-performing algorithms \cite{Choi15, FagotBouquet2016, kim_ICCV2015_MHTR}. We use the same data as in~\cite{tang-2016}, therefore the performance gain is due to our algorithms that solve the full problem~\cite{tang-2015}.

\textbf{Cost and time.}
The convergence of the algorithms \alga{} and \algb{} is shown in Fig.~\ref{fig:mot_time_data}. 
It can be seen from this figure that \alga{} converges faster than \algb{}. 

\textbf{Application-specific accuracy.}
We compare the feasible solutions output by \alga{} and \algb{} to the state-of-the-art for the benchmark~\cite{MilanL0RS16}.
To this end, we report in Tab.~\ref{tab:mot-result} the standard CLEAR MOT metric, including: multiple object tracking accuracy (MOTA), multiple object tracking precision (MOTP), mostly tracked object (MT), mostly lost (ML) and tracking fragmentation (FM). 
MOTA combines identity switches (ID Sw), false positives (FP) and false negatives (FN) and is most widely used. 
Our feasible solutions are published also at the benchmark website unser the names NLLMP (\alga{}) and NLLMPj (\algb{}).
It is can be seen from Tab.~\ref{tab:mot-result} that the feasible solutions obtained by \alga{} and \algb{} rank first in MOTA and MOTP.
Compared to \cite{tang-2016}, \alga{} and \algb{} reduce the number of false positives and false negatives.
The average inverse running time per frame of a video sequence (column ``Hz'' in the table) is better for \alga{} by a margin than for any other algorithm.
Overall, these results show the practicality of the \nllmp{} in conjunction with the local search algorithms \alga{} and \algb{} for applications in multiple object tracking.

\section{Conclusion}

We have stated the minimum cost node labeling lifted multicut problem, \nllmp{}, an \textsc{np}-hard combinatorial optimization problem whose feasible solutions define both a decomposition and a node labeling of a given graph.
We have defined and implemented two local search algorithms, \alga{} and \algb{}, that converge monotonously to a local optimum, offering a feasible solution at any time.
We have shown applications of these algorithms to the tasks of 
articulated human body pose estimation, multiple object tracking and instance-separating semantic segmentation, obtaining state-of-the-art application-specific accuracy.
We conclude that the \nllmp{} is a useful mathematical abstraction in the field of computer vision that allows researchers to apply the same optimization algorithm to diverse computer vision tasks.
To foster collaboration between the fields of computer vision and combinatorial optimization, we make our code publicly available at
\url{https://github.com/bjoern-andres/graph}

{\small
\bibliographystyle{ieee}
\bibliography{andres,manuscript,tracking}}

\appendix

\section{Implementation Details}
\label{app:alg_details}

Func.~\ref{function:update-labeling} and \ref{function:update-two-cut} choose local transformations that decrease the cost optimally.
Our implementation computes cost differences incrementally, as proposed by Kernighan and Lin~\cite{kernighan-1970}.
The exact computations are described below.

\textbf{Transforming the Labeling.}
Func.~\ref{function:update-labeling} repeatedly chooses a node $\hat v$ and a label $\hat l$ such that labeling $\hat v$ with $\hat l$ decreases the cost maximally.
I.e., Func.~\ref{function:update-labeling} repeatedly solves the optimization problem
\begin{align}
(\hat v, \hat l) \in \!\! \underset{(v, l) \in V \times L}{\mathrm{argmin}}
    \varphi(x^{T_{vl}(\lambda_t)}, y^{\mu_t}) - \varphi(x^{\lambda_t}, y^{\mu_t}) 
\enspace .
\label{eq:appendix:tansform-label}
\end{align}
While $\varphi(x^{\lambda_t}, y^{\mu_t})$ is constant, it is more efficient to minimize the difference $\varphi(x^{T_{vl}(\lambda_t)}, y^{\mu_t}) - \varphi(x^{\lambda_t}, y^{\mu_t})$ than to minimize $\varphi(x^{T_{vl}(\lambda_t)}, y^{\mu_t})$, as the difference can be computed locally, considering only the neighbors $w$ of $v$ in $G'$:
\begin{align}
& \varphi(x^{T_{vl}(\lambda_t)}, y^{\mu_t}) - \varphi(x^{\lambda_t}, y^{\mu_t}) \nonumber \\
=\ & c_{vl} - c_{v\lambda_t(v)} \nonumber \\
& + \sum_{vw \in A} (1 - y_{\{v,w\}})\left(c^{\sim}_{vw,l\lambda_t(w)} 
    - c^{\sim}_{vw,\lambda_t(v)\lambda_t(w)}\right) \nonumber \\
& + \sum_{wv \in A} (1 - y_{\{v,w\}})\left(c^{\sim}_{wv,\lambda_t(w)l} 
    - c^{\sim}_{wv,\lambda_t(w)\lambda_t(v)}\right) \nonumber \\
& + \sum_{vw \in A} y_{\{v,w\}}\left(c^{\not\sim}_{vw,l\lambda_t(w)} 
    - c^{\not\sim}_{vw,\lambda_t(v)\lambda_t(w)}\right) \nonumber \\
& + \sum_{wv \in A} y_{\{v,w\}}\left(c^{\not\sim}_{wv,\lambda_t(w)l} 
    - c^{\not\sim}_{wv,\lambda_t(w)\lambda_t(v)}\right) \nonumber \\
=:\ & \Delta_{t,vl}
\enspace .
\end{align}
Initially, i.e., for $t = 0$, we compute $\Delta_{0,vl}$ for every node $v$ and every label $l$.
In subsequent iterations, i.e., for $t \in \mathbb{N}$ and the minimizer $(\hat v, \hat l)$ of \eqref{eq:appendix:tansform-label} chosen in this iteration, we update cost differences for all neighbors $w$ of $\hat v$ in $G'$ and all labels $l \in L$.
The update rule is written below for an edge $(w, \hat v) \in A$.
The update for an edge in the opposite direction is analogous.
Below, \eqref{eq:alt_cancel} subtracts the costs due to $\hat v$ being labeled $\lambda_t(\hat v)$ (which is possibly outdated), while \eqref{eq:alt_update} adds the costs due to $\hat v$ having obtained a new and possibly different label $\hat l$.
\begin{align}
\Delta_{t+1, wl} = 
& \Delta_{t, wl} \nonumber \\
& - (1 - y_{\{w, \hat v\}})\left(c^{\sim}_{w\hat v,l\lambda_t(\hat v)} - c^{\sim}_{w\hat v,\lambda_t(w)\lambda_t(\hat v)}\right) \nonumber \\
& - y_{\{w, \hat v\}}\left(c^{\not\sim}_{w\hat v,l\lambda_t(\hat v)} - c^{\not\sim}_{w\hat v,\lambda_t(w)\lambda_t(\hat v)}\right) \label{eq:alt_cancel} \\ 
& + (1 - y_{\{w, \hat v\}})\left(c^{\sim}_{w\hat v,l\hat l} - c^{\sim}_{w\hat v,\lambda_t(w)\hat l}\right) \nonumber \\ 
& + y_{\{w, \hat v\}} \left(c^{\not\sim}_{w\hat v,l\hat l} - c^{\not\sim}_{w\hat v,\lambda_t(w)\hat l}\right) \enspace . \label{eq:alt_update}
\end{align}

We solve \eqref{eq:appendix:tansform-label} by means of a priority queue in time complexity 
$\mathcal{O}(|V| \log|V| + |V| (|L| + \log|V|) \deg G')$ with $\deg G'$ the node degree of $G'$.
For sparse graphs and constant number $|L|$ of labels, this is $\mathcal{O}(|V| \log|V|)$.

\textbf{Transformation of Labeling and Decomposition.}
The algorithm KLj of \cite{keuper-2015a} for the minimum cost lifted multicut problem generalizes the Kernighan-Lin-Algorithm~\cite{kernighan-1970} for the minimum cost multicut problem.
The algorithms \alga{} and \algb{} we define further generalize KLj to the \nllmp{}.
The critical part is Func.~\ref{function:update-two-cut} that solves the optimization problem
\begin{align}
    (\hat v, \hat l) \in \underset{(v, l) \in V_t \times L}{\mathrm{argmax}} \varphi(x^{T_{vl}(\lambda_t)}, y^{T'_{v m'}(\mu_t)}) - \varphi(x^\lambda_t, y^{\mu_t})
\end{align}
Let us consider w.l.o.g. two sets of vertices $A$ and $B$ representing two neighboring components of the graph $G$.
Then we compute $\forall v \in A \cup B, \forall l \in L$:

\begin{align}
\Delta_{vl} &=  c_{v\lambda_t(v)} - c_{vl} \, + \\
&\begin{cases}
\sum_{w \in A \setminus \{v\}} & c^{\sim}_{vw,\lambda_t(v)\lambda_t(w)} - c^{\not\sim}_{vw,l\lambda_t(w)} \\
\sum_{w \in B} & c^{\not\sim}_{vw,\lambda_t(v)\lambda_t(w)} - c^{\sim}_{vw,l\lambda_t(w)} \\
\sum_{w \notin A \cup B} & c^{\not\sim}_{vw,\lambda_t(v)\lambda_t(w)} - c^{\not\sim}_{vw,l\lambda_t(w)} \label{eq:gains_precompute}
\end{cases} \enspace ,
\end{align}
where $w \in \mathcal{N}_{G'}(v)$.
In eq.~\eqref{eq:gains_precompute} the first two cases are exactly the same as given in~\cite{kernighan-1970} for the edges \emph{between} partitions $A$ and $B$.
But in our case changing vertex's class label may affect the cut costs of edges between $A$ and $B$ and any other partition.
Also, we have join and cut costs.

Let us assume w.l.o.g. that vertex $\hat v$ was chosen to be moved from set $A$ to set $B$, i.e. $A = A \setminus \{\hat v\}$.
Now we can update the expected gains of $\forall w \in \mathcal{N}_{G'}(\hat v), \forall l \in L$:

\begin{align}
\Delta_{wl} = \Delta_{wl} &- \left(c^{\sim}_{\hat vw,\lambda_t(\hat v)\lambda_t(w)} - c^{\not\sim}_{\hat vw,\lambda_t(\hat v)l}\right) \nonumber \\
&+ c^{\not\sim}_{\hat vw,\hat l\lambda_t(w)} - c^{\sim}_{\hat vw,\hat ll} \enspace , \quad \text{if $w \in A$} \enspace , \\
\Delta_{wl} = \Delta_{wl} &- \left(c^{\not\sim}_{\hat vw,\lambda_t(\hat v)\lambda_t(w)} - c^{\sim}_{\hat vw,\lambda_t(\hat v)l}\right) \nonumber \\
&+ c^{\sim}_{\hat vw,\hat l\lambda_t(w)} - c^{\not\sim}_{\hat vw,\hat ll} \enspace , \quad \text{if $w \in B$} \enspace .
\end{align}
After that $B = B \cup \{\hat v\}$.
In the above equations, the expression in parenthesis cancels the current contribution for vertex $w$, that assumed $\hat v$ was labeled $\lambda_t(v)$ and belonged to partition $A$. For the case when $|L| = 1$ and $c^{\not\sim} = c^{\sim}$ the above equations simplify to exactly the ones as in~\cite{kernighan-1970}, but multiplied by 2, because in our objective we have two terms that operate on the edges simultaneously.

As we generalize \cite{keuper-2015a} by an additional loop over the set $L$ of labels, the analysis of the time complexity carries over from \cite{keuper-2015a} with an additional multiplicative factor $|L|$.

\section{Articulated Human Body Pose Estimation}
\label{app:pose}

\subsection{Problem Statement}

Pishchulin et al.~\cite{pishchulin-2016} introduce a binary cubic problem w.r.t.~a set $C$ of body joint classes and a set $D$ of putative detections of body joints.
Every feasible solution is a pair $(x,y)$ with $x: D \times C \to \{0,1\}$ and $y: \tbinom{D}{2} \to \{0,1\}$, constrained by the following system of linear inequalities:
\begin{align}
\forall d \in D \forall cc' \in \tbinom{C}{2}:
    & \quad x_{dc} + x_{dc'} \leq 1
    \label{eq:constraint-map}\\
\forall dd' \in \tbinom{D}{2}:
    & \quad y_{dd'} \leq \sum_{c \in C} x_{dc}
    \nonumber\\
    & \quad y_{dd'} \leq \sum_{c \in C} x_{d'c}
    \label{eq:constraint-suppression-1}\\
\forall dd'd'' \in \tbinom{D}{3}:
    & \quad y_{dd'} + y_{d'd''} - 1 \leq y_{dd''}
    \label{eq:constraint-cycle}
\end{align}
The objective function has the form below with coefficients $\alpha$ and $\beta$.
\begin{align}
\sum_{d \in D} \sum_{c \in C} 
    \alpha_{dc} \, x_{dc}
+ \sum_{d d' \in \tbinom{D}{2}} \sum_{c,c' \in C} 
    \beta_{dd'cc'} x_{dc} x_{d'c'} y_{dd'}
\end{align}

We identify the solutions of this problem with the solutions of the \nllmp{} w.r.t.~the complete graphs $G = G' = (D, \tbinom{D}{2})$, 
the label set $L = C \cup \{\epsilon\}$
and the costs $c^{\not\sim} = 0$ and
\begin{align}
c_{vl} & := \begin{cases}
    \alpha_{vl} & \textnormal{if}\ l \in C \\
    0 & \textnormal{if}\ l = \epsilon \\
\end{cases} \\
c^\sim_{vw,ll'} & := \begin{cases}
    \beta_{vwll'} & \textnormal{if}\ l \in C \wedge  l' \in C \\
    0 & \textnormal{if}\ l = \epsilon \ \textnormal{xor}\ l' = \epsilon \\
    \infty & \textnormal{if}\ l = l' = \epsilon
\end{cases}
\enspace .
\end{align}
Note that in \cite{pishchulin-2016}, $y_{dd'} = 1$ indicates a join.
In our \nllmp{}, $y_{dd'} = 1$ indicates a cut.

\subsection{Further Results}

Quantitative results for each body joint are shown in Tab.~\ref{tab:detailed_pose_results}.
Qualitative results for the MPII Human Pose dataset are shown in Fig.~\ref{fig:qualitative_mpii_multi}.

\addtolength{\tabcolsep}{-2pt}
\begin{table}
\centering
\small
\begin{tabular}{@{}clrrrrrrrr}
\toprule
$|V|$ & Alg. & Head & Sho & Elb & Wri & Hip & Knee & Ank & AP \\
\midrule
& \cite{insafutdinov-2016} & 84.9  & 79.2  & 66.4  & 52.3  & 65.5  & 59.2 & 51.2 & 65.5 \\
\cmidrule{2-10}
& \alga{} & \bf 87.1  & 80.0  & \bf 66.8  & \bf 53.6  & 66.1  & 60.0 & 51.8 & 66.5 \\
\multirow{-3}{*}{\begin{sideways}150\end{sideways}} & \algb{} & 86.8  & \bf 80.2  & 67.5  & 53.5  & \bf 66.3  & \bf 60.3 & \bf 51.9 & \bf 66.6 \\
\midrule
& \alga{} & \bf 90.2  & \bf 85.2  & 71.5  & 59.5  & \bf 71.3  & \bf 63.1 & 53.1 & \bf 70.6 \\
\multirow{-2}{*}{\begin{sideways}420\end{sideways}} & \algb{} & 89.8  & \bf 85.2  & \bf 71.8  & \bf 59.6  & 71.1  & 63.0 & \bf 53.5 & \bf 70.6 \\
\bottomrule
\end{tabular}
\caption{Comparison of \textsc{b\&c}~\cite{insafutdinov-2016}, \alga{} and \algb{} in an application to the task of articulated human body pose estimation.}
\label{tab:detailed_pose_results}
\end{table}
\addtolength{\tabcolsep}{2pt}

\begin{figure*}
  \centering
  \begin{tabular}{c c c c}
    \includegraphics[height=0.153\linewidth]{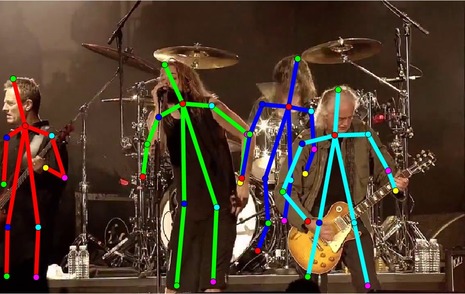} &
    \includegraphics[height=0.153\linewidth]{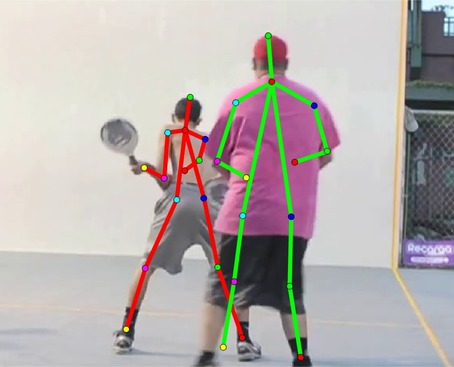} &
    \includegraphics[height=0.153\linewidth]{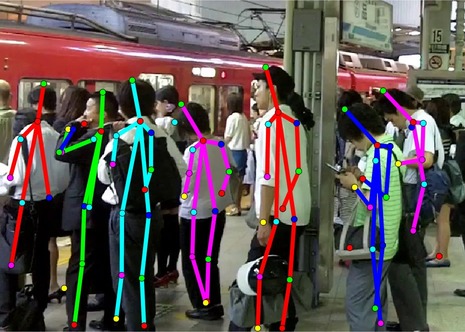} &
    \includegraphics[height=0.153\linewidth]{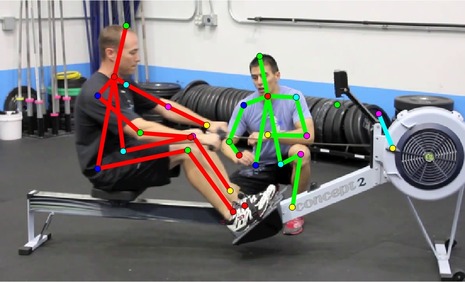} \\
    1&2&3&4\\
  \end{tabular}

  \begin{tabular}{c c c}
    \includegraphics[height=0.17\linewidth]{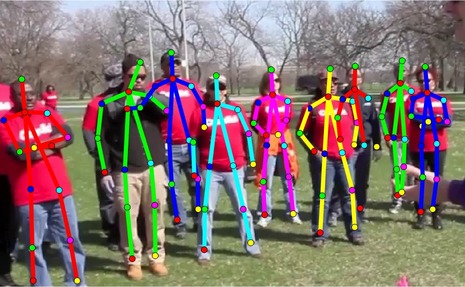} &
    \includegraphics[height=0.17\linewidth]{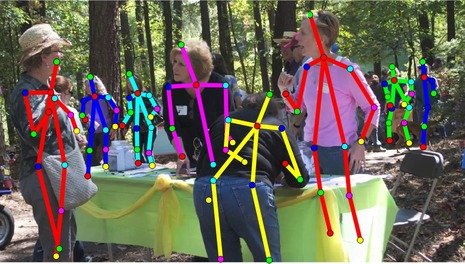} &
    \includegraphics[height=0.17\linewidth]{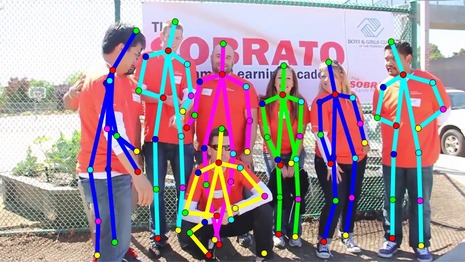} \\
    5&6&7\\
  \end{tabular}
  
  \begin{tabular}{c c c c}
    \includegraphics[height=0.16\linewidth]{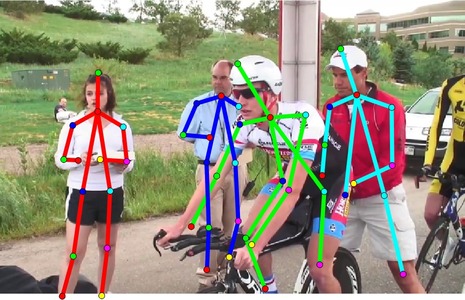} &
    \includegraphics[height=0.16\linewidth]{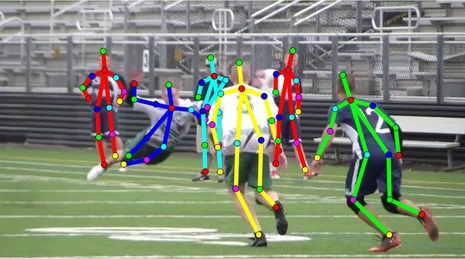} &
    \includegraphics[height=0.16\linewidth]{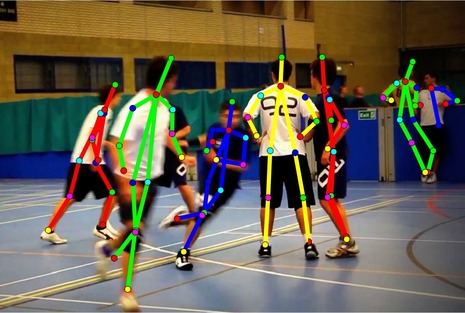} &
    \includegraphics[height=0.16\linewidth]{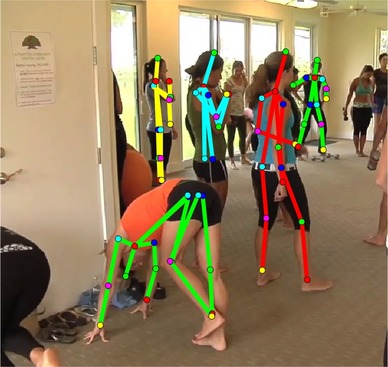} \\
    8&9&10&11\\
  \end{tabular}

  \begin{tabular}{c c c c}
    \includegraphics[height=0.163\linewidth]{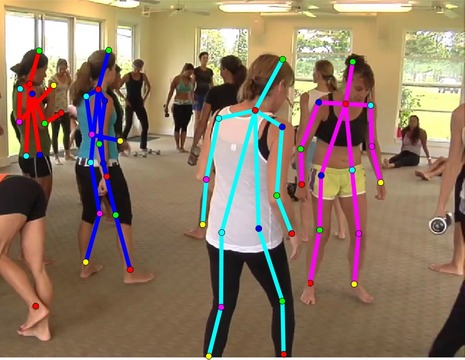} &
    \includegraphics[height=0.163\linewidth]{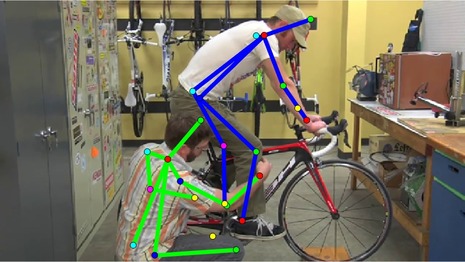} &
    \includegraphics[height=0.163\linewidth]{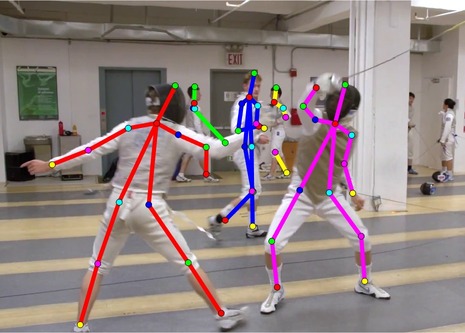} &
    \includegraphics[height=0.163\linewidth]{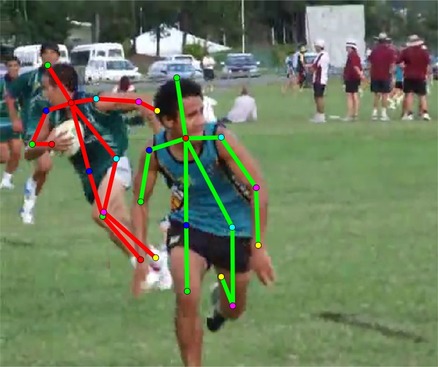} \\
    12&13&14&15\\
  \end{tabular}

  \begin{tabular}{c c c}
    \includegraphics[height=0.186\linewidth]{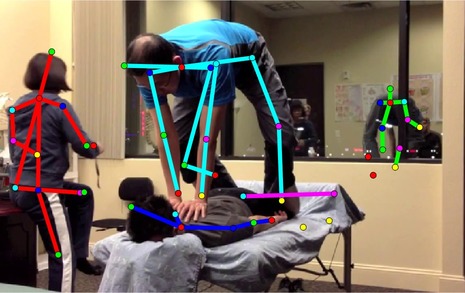} &
    &
    \\
    16 & &\\
  \end{tabular}
  
  \caption{Pose estimation results on the MPII Human Pose dataset.}
  \label{fig:qualitative_mpii_multi}
\end{figure*}

\section{\mbox{Instance-Separating Semantic Segmentation}}
\label{app:instance}

We tackle the problem of instance-separating semantic segmentation by
adapting the approach of Uhrig et al.~\cite{uhrig16gcpr}. They propose three complementary representations,
which are learned jointly by a fully convolutional network (FCN)~\cite{fcn15cvpr}, that facilitate the
problem of separating individual objects: Semantics, depth, and directions towards object centers. To extract
object instances, a template matching approach was proposed, followed by a proposal fusion.

Instead of template matching and clustering, we rely on a generic graphical formulation of the problem using
only the three predicted output scores from the network of Uhrig et al.~\cite{uhrig16gcpr}, together with a
suitable formulation of unary $c$ and pairwise terms $c^\sim$ and $c^\nsim$. As there might be up to two million
nodes for a direct mapping of pixel scores to the graph, we report performance on different down-sampled
versions to reduce overall computation time and reduce the impact of noise in high resolutions. Results on
KITTI were achieved on half of the input resolution, for Cityscapes we down-sample the FCN scores by a factor
of eight before the graph optimization.

\subsection{Cut Costs Details}

To define cut costs between connected pixels in the graph, we use an equally weighted sum of the three following
components:

The probability of fusing two pixels $v$ and $w$ of different \textbf{semantic classes} is $1-p(\lambda(v)=a, \lambda(w)=b)$,
the probability of confusing label class $a$ and $b$, which was computed from the training set.

To incorporate the depth and center direction channels, we neither use scores nor argmax predictions directly.
Instead, we weight the predicted softmax scores for all non-background classes with their corresponding class to
recover a continuous center direction and depth map. As objects at different \textbf{depth values} should be
separated, we generate higher cut probabilities for those pixels. From training data, we found the probability
of splitting two neighboring pixels to be one when the predicted depth values differ by more than 1.6 units.

If \textbf{center directions} have opposite orientations, there should be a high probability for splitting the
two pixels. However, opposite directions also appear at the center of an object. Therefore, we define the cut
probability as the minimum of an angular inconsistency, which punishes two pixels that point at different directions,
as well as a center inconsistency, which punishes if two pixels do not point at each other, \cf
Fig.~\ref{figure:direction_consistency}. This induces high cut probabilities at the borders of objects, as directions
of pixels should have opposite center direction. The probability of splitting two neighbors due to direction
inconsistency was found to be one at 90 degrees.

\begin{figure}
\centering
\includegraphics[width=0.35\textwidth]{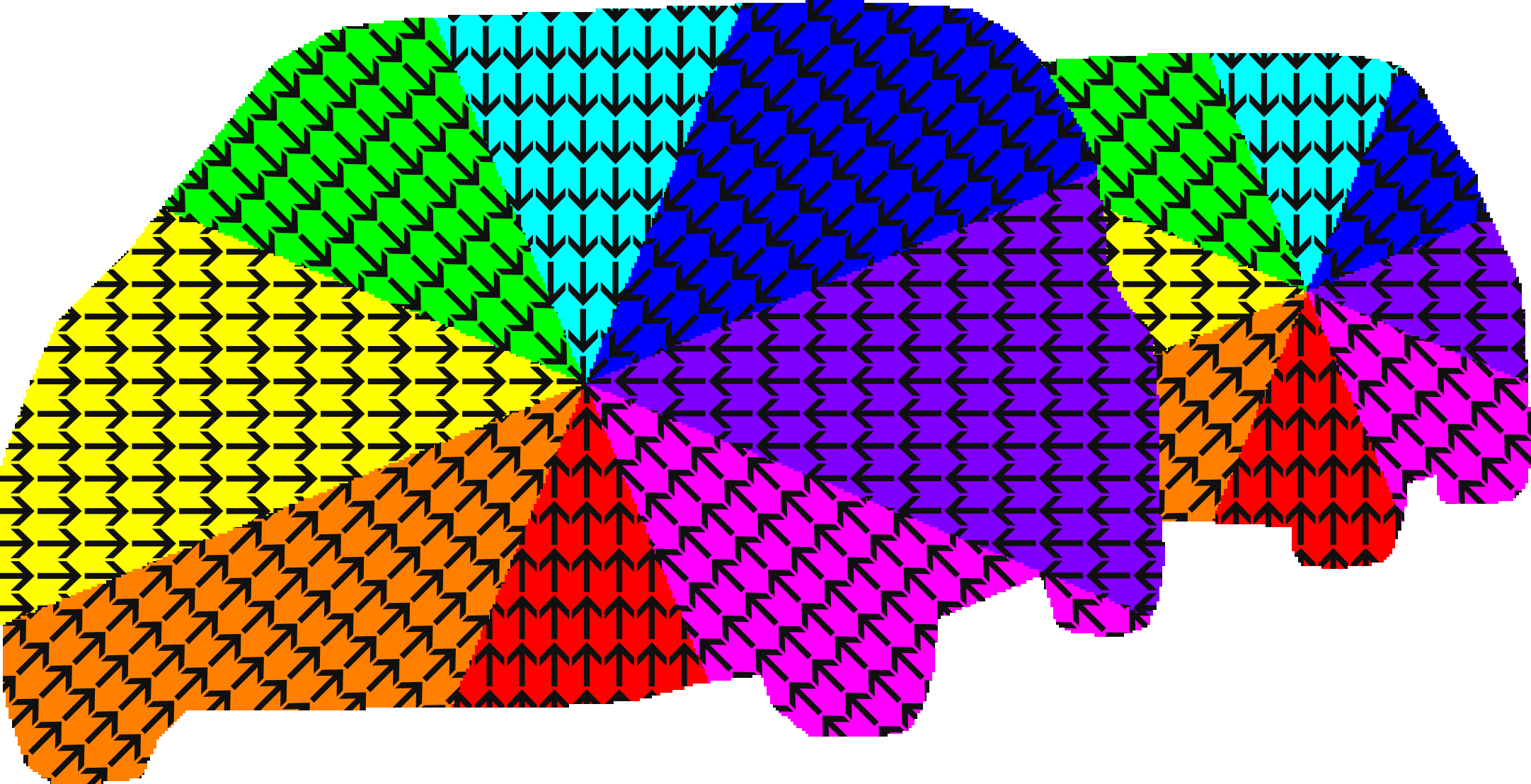}
\caption{Instance center directions with color coding from~\cite{uhrig16gcpr}. Near object centers, directions
point towards each other (center consistency). Within colored regions, directions have similar angles
(angular consistency). Along object borders, directions are inconsistent in both ways.}
\label{figure:direction_consistency}
\end{figure}

\subsection{Dataset Specifics}

For the KITTI dataset \cite{kitti14instances,kitti12cvpr}, the only pixel-level annotated object class
is \textit{car}. For Cityscapes \cite{cityscapes16cvpr} however, there are 8 different object classes
(\textit{person, rider, car, truck, bus, train, motorcycle, bicycle}), together with 11 background classes
versus 1 background class for KITTI. We found that the network model used by Uhrig et al. \cite{uhrig16gcpr}
performs close to optimum for semantic labeling on KITTI data, however has some flaws on Cityscapes.

Therefore we chose a more sophisticated network structure, which performs much better on the many different
classes on Cityscapes. We use a ResNet \cite{He_2016_CVPR} with dilated
convolutions \cite{CP2016Deeplab} for cut costs $c$, namely the unary terms consisting of scores for the problem
of semantic labeling, which was trained independently on the Cityscapes dataset~\cite{cityscapes16cvpr}.

To obtain the unaries for Cityscapes, we use a slightly modified ResNet-50 network. We introduce dilated convolutions in the conv4\_x and conv5\_x layers to increase the output resolution from $\nicefrac{1}{32}$ to $\nicefrac{1}{8}$ of the input resolution. We then remove the final average pooling layer and for classification use a convolutional layer with $5 \times 5$ dilated kernels with a dilation size of 12. This is identical to the best performing basic ResNet-50 variant reported in (\cite{WuSH16b}, Table 1).

Due to GPU memory constraints, we train with $512px \times 768px$ crops randomly sampled from the full-resolution training set images. We apply minimal data augmentation, i.e. random horizontal flips, and train with a batch size of 5. We train the network for 60000 iterations using the Adam solver with an initial learning rate of $0.000025$, weight decay of $0.0005$ and momentum of $0.9$. We use the "poly" learning rate policy to gradually reduce the learning rate during training with the power parameter set to $0.9$, which as reported in both \cite{Liu2015ParseNetLW} and \cite{ChenPK0Y16} yields better results than the commonly used "step" reduction policy.

At test-time we apply the network to overlapping $1024px \times 768px$ crops of the full-resolution test set images and stitch the results to obtain the final predictions.

For KITTI however, we stick with the original semantic scores. The only adaptation for our definition of
the semantic cut costs $c$ is an additional weighting of the semantic scores: As depth and center directions
are only estimated for objects, all three channels contain knowledge of the objectness of a certain pixel.
We therefore use the semantic scores weighted by the depth and direction scores for objects as unaries.
This increases robustness of the semantics as all three channels must agree to achieve high scores.

\subsection{Post Processing}

Using the unary and pairwise terms defined above, we solve the graph for labels and components with our proposed
algorithms \alga{} and \algb{}. As the center direction representation inherently cannot handle cases of full
occlusions, e.g. if a bicycle is split into two connected components by a pedestrian in front of it, we apply a
similar component fusion technique as proposed in~\cite{uhrig16gcpr}: We accumulate direction predictions within
each component and fuse it with another suitable component when direction predictions are clearly overshooting
into a certain direction. We compare performance of the raw graph output as well as the fused instances in
Tab.~\ref{tab:instance_results_CS_KITTI} (top).

\begin{table}
\setlength{\tabcolsep}{4pt}
\renewcommand{\arraystretch}{1}
\centering
\small
\begin{tabular}{l@{\hskip 12pt}ccc}
\toprule
Algorithm                         & Dataset     & $\mapr$       & $\mapr^{50\%}$    \\
\midrule
Ours \alga{} (raw)                & KITTI val   & $43.0$        & $72.5$            \\
Ours \algb{} (raw)                & KITTI val   & $43.5$        & $72.6$            \\
Ours \alga{} (fused)              & KITTI val   & \bst{50.5}    & \bst{82.9}        \\
Ours \algb{} (fused)              & KITTI val   & $50.3$        & $82.4$            \\
\midrule
Pixel Encoding \cite{uhrig16gcpr} & KITTI test  & $41.6$        & $69.1$            \\
Ours \algb{} (fused)              & KITTI test  & \bst{43.6}    & \bst{71.4}        \\
\bottomrule
\end{tabular}
\caption{
Comparison of algorithms for instance segmentation on the KITTI~\cite{kitti14instances} datasets
using the mean average precision metrics introduced in \cite{cityscapes16cvpr}.
}
\label{tab:instance_results_CS_KITTI}
\end{table}

\subsection{Detailed Results}

As there are different metrics used by related approaches, we report performance on the Cityscapes~\cite{cityscapes16cvpr}
and KITTI~\cite{kitti14instances,kitti12cvpr} dataset using both proposed metrics. The instance score required for the evaluation
on Cityscapes was chosen as the size of the instance in pixels multiplied by its mean depth - this score achieved slightly better
results compared to a constant score.

For KITTI, we outperform all existing approaches using the Cityscapes metric (without adapting the semantic scores of
Uhrig et al. \cite{uhrig16gcpr}), which averages precision and recall performance
for multiple overlaps, \cf Tab.~\ref{tab:instance_results_CS_KITTI} (bottom). We evaluate the performance using \alga{} or \algb{}
and raw graph output (raw) or the post-processed results using above described fusion (fused) in
Tab.~\ref{tab:instance_results_CS_KITTI} (top). Using the KITTI metrics,
we perform among the best results while having a slight preference of Recall over Precision, \cf
Tab.~\ref{tab:instance_results_KITTI}.

\begin{table}
\setlength{\tabcolsep}{5pt}
\renewcommand{\arraystretch}{1}
\centering
\small
\begin{tabular}{l@{\hskip 12pt}cccccc}

\toprule
Alg.                    & IoU        & AvgFP       & AvgFN       & InsPr      & InsRe      & InsF1      \\
\midrule
\cite{zhang2015iccv}    & $77.4$     & $0.479$     & $0.840$     & $48.9$     & $43.8$     & $46.2$     \\
\cite{zhang2016cvpr}    & $77.0$     & $0.375$     & $1.139$     & $65.3$     & $50.0$     & $56.6$     \\
\cite{uhrig16gcpr}      & $84.1$     & $0.201$     & $0.159$     & \bst{86.3} & $74.1$     & \bst{79.7} \\
\cite{ren16arxiv}       & \bst{87.4} & \bst{0.118} & $0.278$     & -          & -          & -          \\
Ours                    & $83.9$     & $0.555$     & \bst{0.111} & $69.2$     & \bst{76.5} & $72.7$     \\
\bottomrule
\end{tabular}
\caption{
Comparison of algorithms for instance segmentation on the KITTI test dataset~\cite{kitti14instances} using metrics proposed in \cite{zhang2016cvpr}. Ours describes the performance of our \algb{} variant.
}
\label{tab:instance_results_KITTI}
\end{table}

For Cityscapes, we report evaluation metrics using both the raw scores of Uhrig et al. \cite{uhrig16gcpr} as well
as our final proposed model using the semantic scores of a ResNet \cite{He_2016_CVPR} together with the center direction
and depth scores of Uhrig et al. \cite{uhrig16gcpr}, \cf Tab.~\ref{tab:instance_results_CS_CS} (top).
Using our adapted ResNet version, we outperform the currently published state-of-the art, \cf
Tab.~\ref{tab:instance_results_CS_CS} (bottom). Note that we report significantly better performance for the large
vehicle classes truck, bus, and trains despite
starting from the same FCN output, \cf Tab.~\ref{tab:class_results_CS}.
This comes from incorporating confusion probabilities between unreliable classes as well as optimizing jointly for semantics
and instances.

\begin{table}
\setlength{\tabcolsep}{3pt}
\renewcommand{\arraystretch}{1}
\centering
\small
\begin{tabular}{l@{\hskip 12pt}cccccccc}

\toprule
& \rotatedlabel{person}
& \rotatedlabel{rider}
& \rotatedlabel{car}
& \rotatedlabel{truck}
& \rotatedlabel{bus}
& \rotatedlabel{train}
& \rotatedlabel{motorcycle}
& \rotatedlabel{bicycle} \\
\midrule
\cite{cityscapes16cvpr} &   $ 5.6$   &   $ 3.9$   &   $26.0$   & $13.8$     & \bst{26.3} & $15.8$     &   $ 8.6$   &   $ 3.1$    \\
\cite{uhrig16gcpr}      & \bst{31.8} & \bst{33.8} & $37.8$     &   $ 7.6$   &   $12.0$   &   $ 8.5$   & $20.5$     & \bst{17.2}  \\
Ours                    &  $18.4$    & $29.5$     & \bst{38.3} & \bst{16.1} & $21.5$     & \bst{24.5} & \bst{21.4}  & $16.0$     \\
\bottomrule
\end{tabular}
\caption{
Comparison of performance on Cityscapes test using the mean average precision metric $\mapr^{50\%}$~\cite{cityscapes16cvpr}.
Ours describes the performance of our \algb{} (ResNet) variant.
}
\label{tab:class_results_CS}
\end{table}

\subsection{Qualitative Results}

See Fig.~\ref{fig:csQualitativeResults} for some qualitative results for our instance-separating semantic segmentation on the
Cityscapes validation dataset~\cite{cityscapes16cvpr}. It can be seen that we perform equally well for large and small objects,
we only tend to fuse pedestrians too often, which explains the worse performance on pedestrians - \cf the mother with her child
on the right in the last row of Fig.~\ref{fig:csQualitativeResults}. Also, the impact of the proposed post-processing based on
the fusion algorithm proposed by Uhrig et al.~\cite{uhrig16gcpr} can be seen clearly: Due to noisy predictions, the raw graph
output is often highly over-segmented. However, after applying the fusion step, most objects are correctly fused.

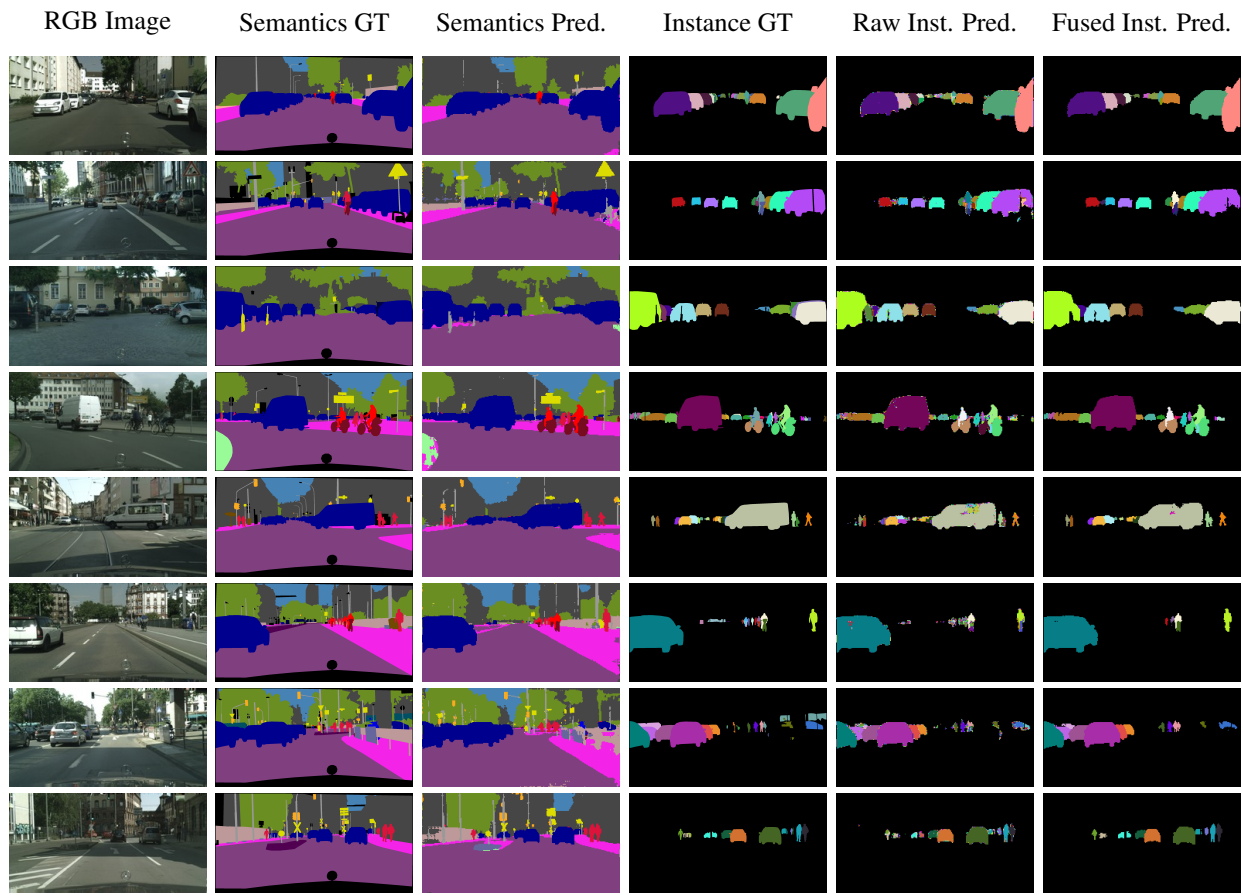
\begin{figure*}
\begin{center}
  \begin{tikzpicture}[scale=.25]
    \node at ( 0,4)    {RGB Image};
    \node at (10,4)    {Semantics GT};
    \node at (20,4)    {Semantics Pred.};
    \node at (30,4)    {Instance GT};
    \node at (40,4)    {Raw Inst. Pred.};
    \node at (50,4)    {Fused Inst. Pred.};

    \node at ( 0,0)    {\includegraphics[width=0.15\textwidth]{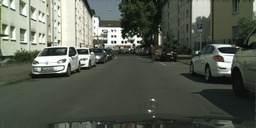}};
    \node at (10,0)    {\includegraphics[width=0.15\textwidth]{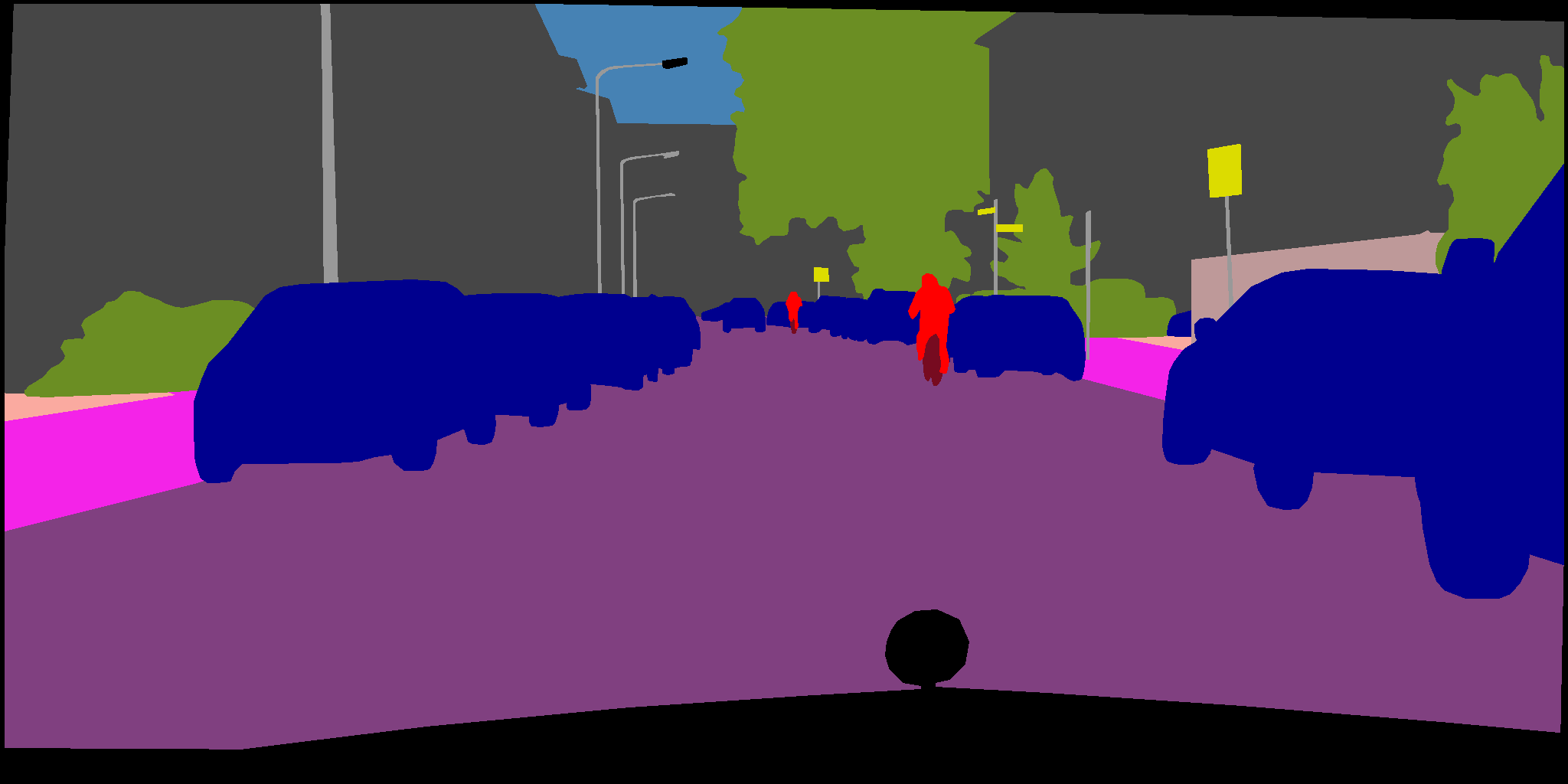}};
    \node at (20,0)    {\includegraphics[width=0.15\textwidth]{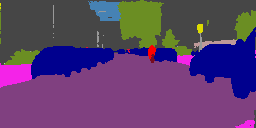}};
    \node at (30,0)    {\includegraphics[width=0.15\textwidth]{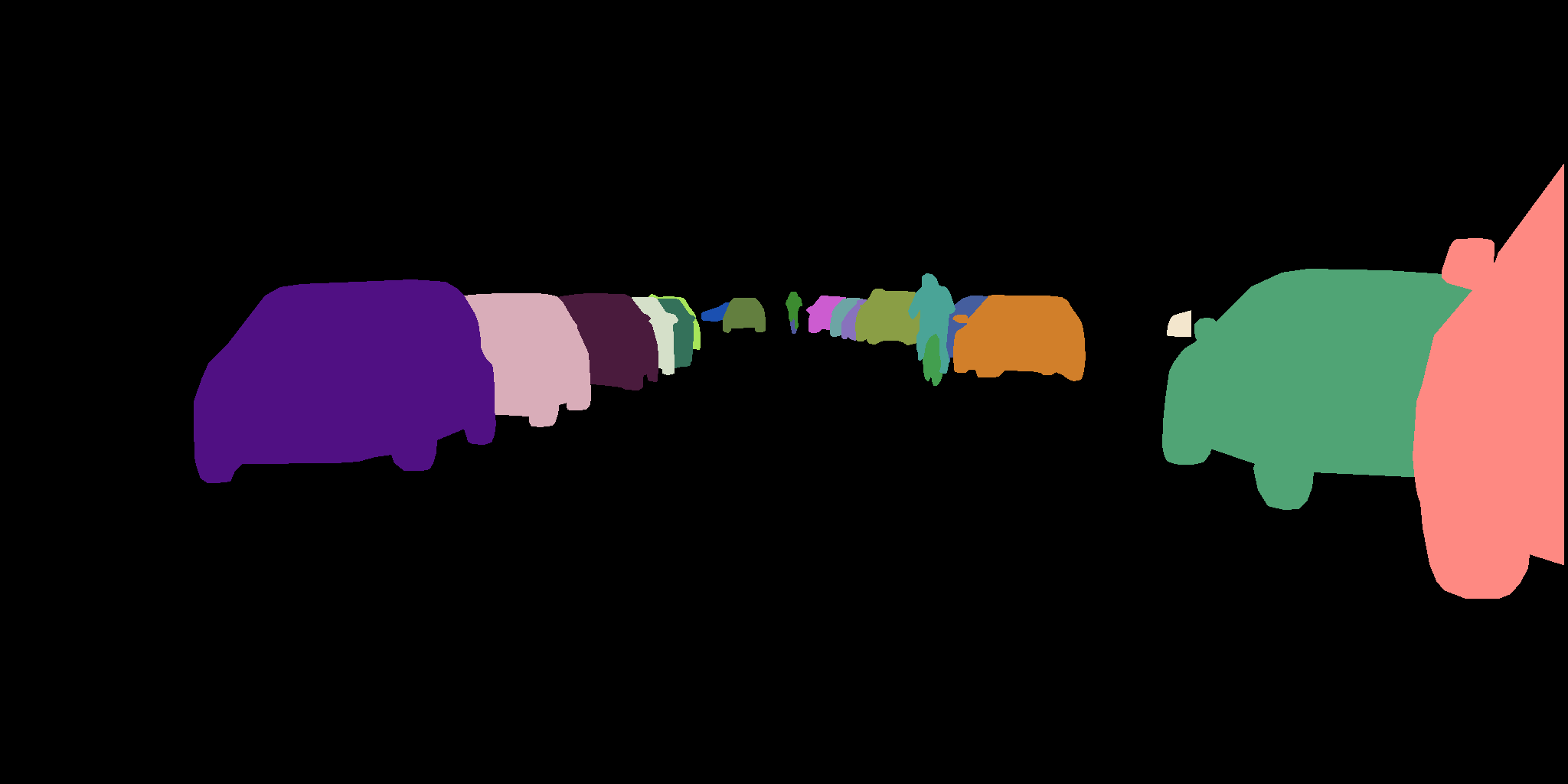}};
    \node at (40,0)    {\includegraphics[width=0.15\textwidth]{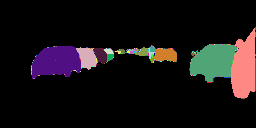}};
    \node at (50,0)    {\includegraphics[width=0.15\textwidth]{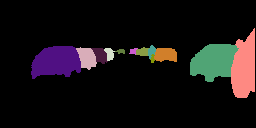}};

    \node at ( 0,-5.1) {\includegraphics[width=0.15\textwidth]{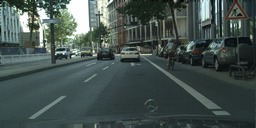}};
    \node at (10,-5.1) {\includegraphics[width=0.15\textwidth]{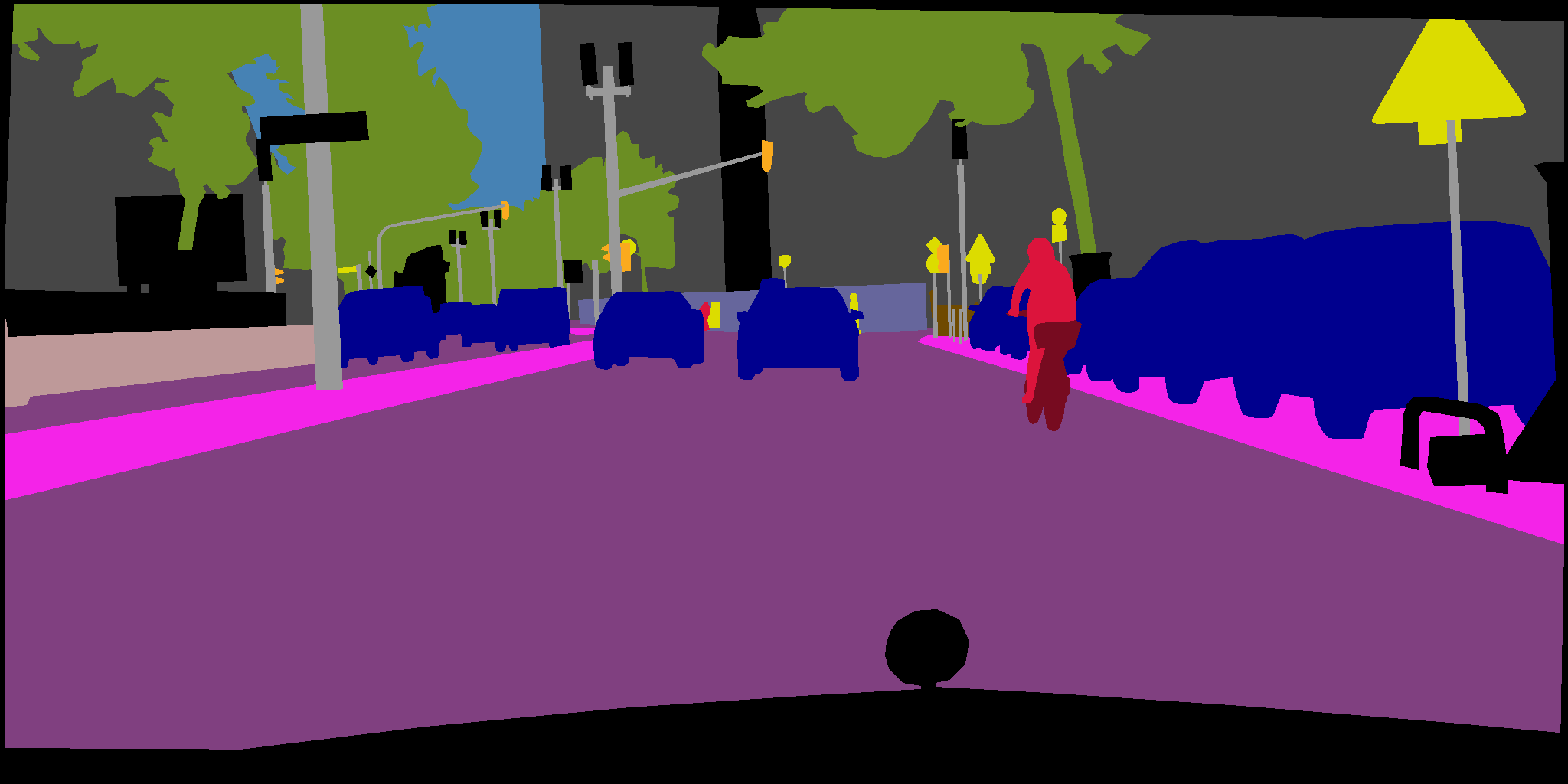}};
    \node at (20,-5.1) {\includegraphics[width=0.15\textwidth]{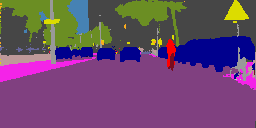}};
    \node at (30,-5.1) {\includegraphics[width=0.15\textwidth]{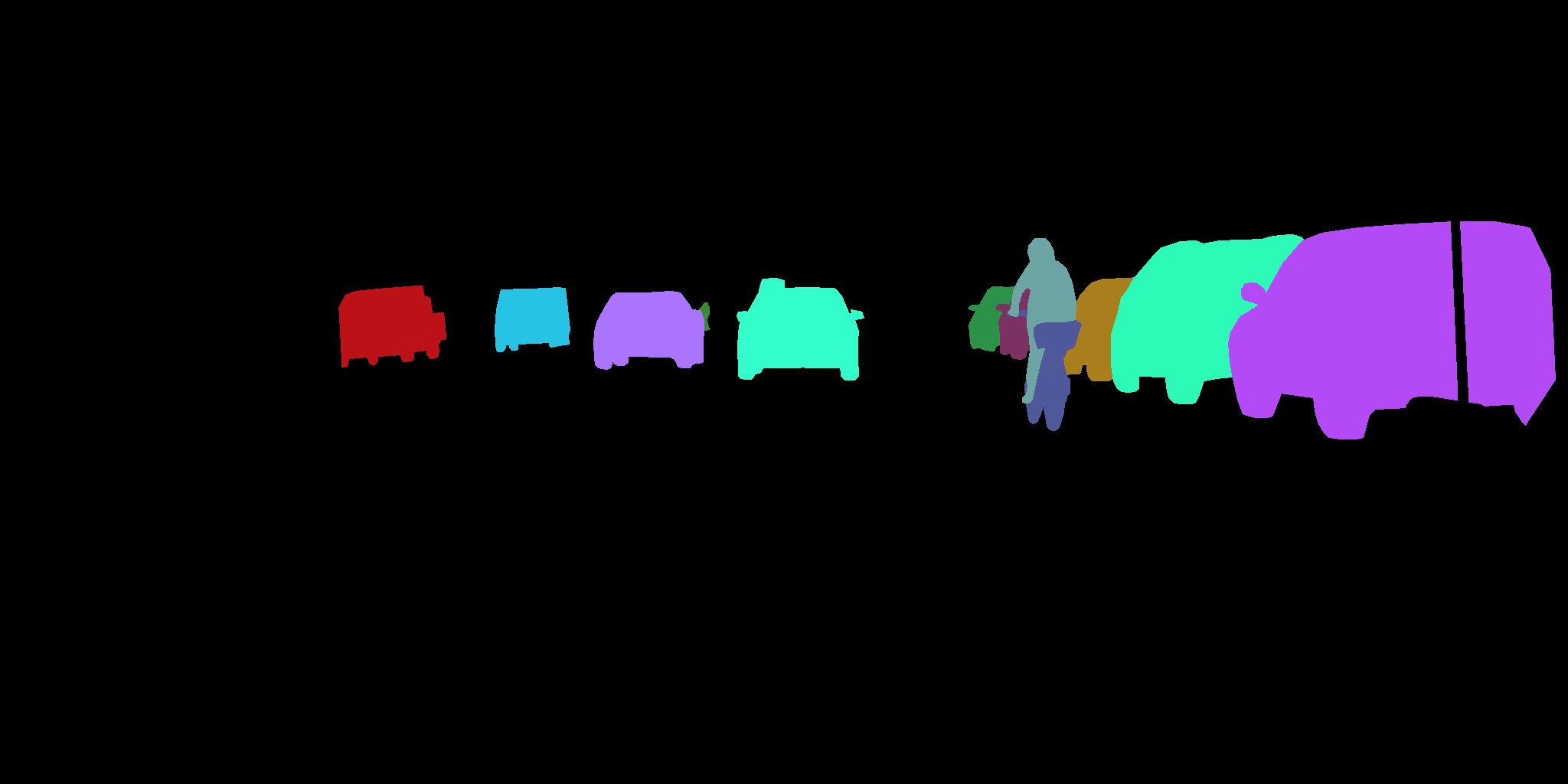}};
    \node at (40,-5.1) {\includegraphics[width=0.15\textwidth]{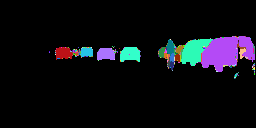}};
    \node at (50,-5.1) {\includegraphics[width=0.15\textwidth]{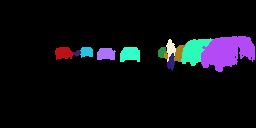}};

    \node at ( 0,-10.2){\includegraphics[width=0.15\textwidth]{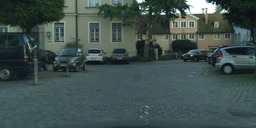}};
    \node at (10,-10.2){\includegraphics[width=0.15\textwidth]{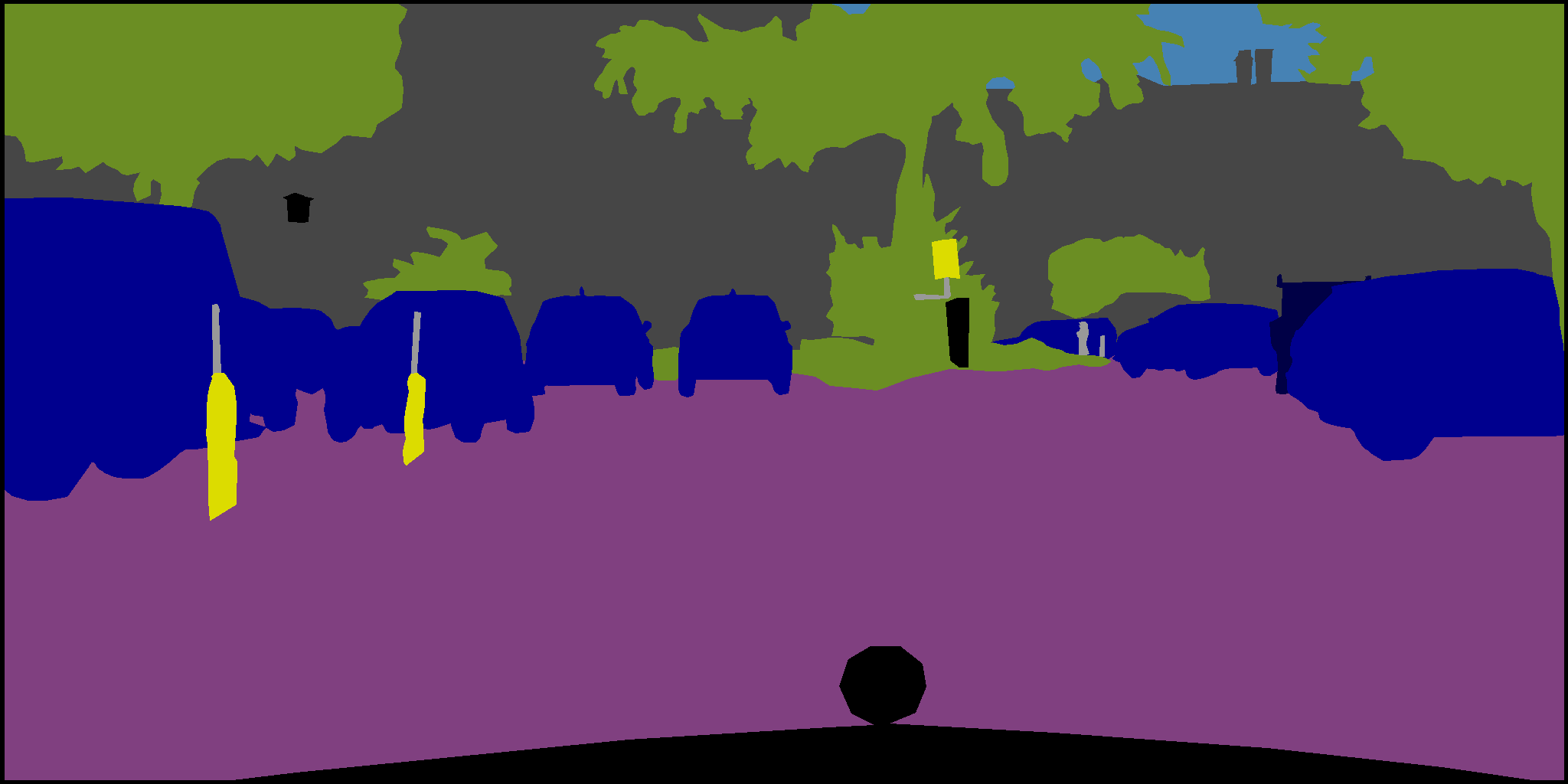}};
    \node at (20,-10.2){\includegraphics[width=0.15\textwidth]{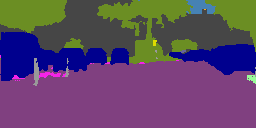}};
    \node at (30,-10.2){\includegraphics[width=0.15\textwidth]{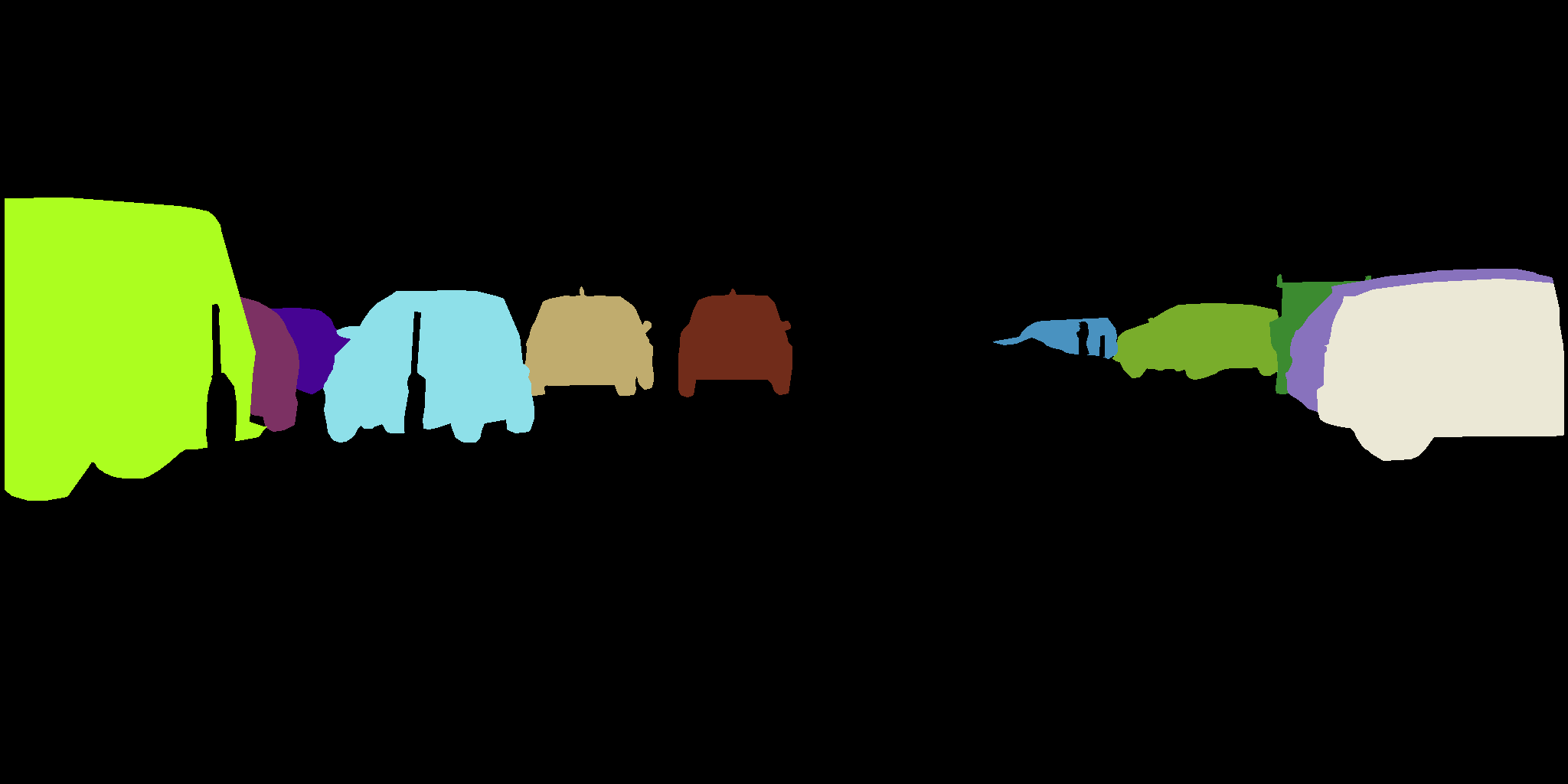}};
    \node at (40,-10.2){\includegraphics[width=0.15\textwidth]{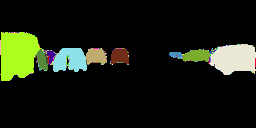}};
    \node at (50,-10.2){\includegraphics[width=0.15\textwidth]{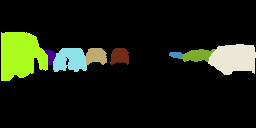}};

    \node at ( 0,-15.3){\includegraphics[width=0.15\textwidth]{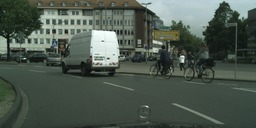}};
    \node at (10,-15.3){\includegraphics[width=0.15\textwidth]{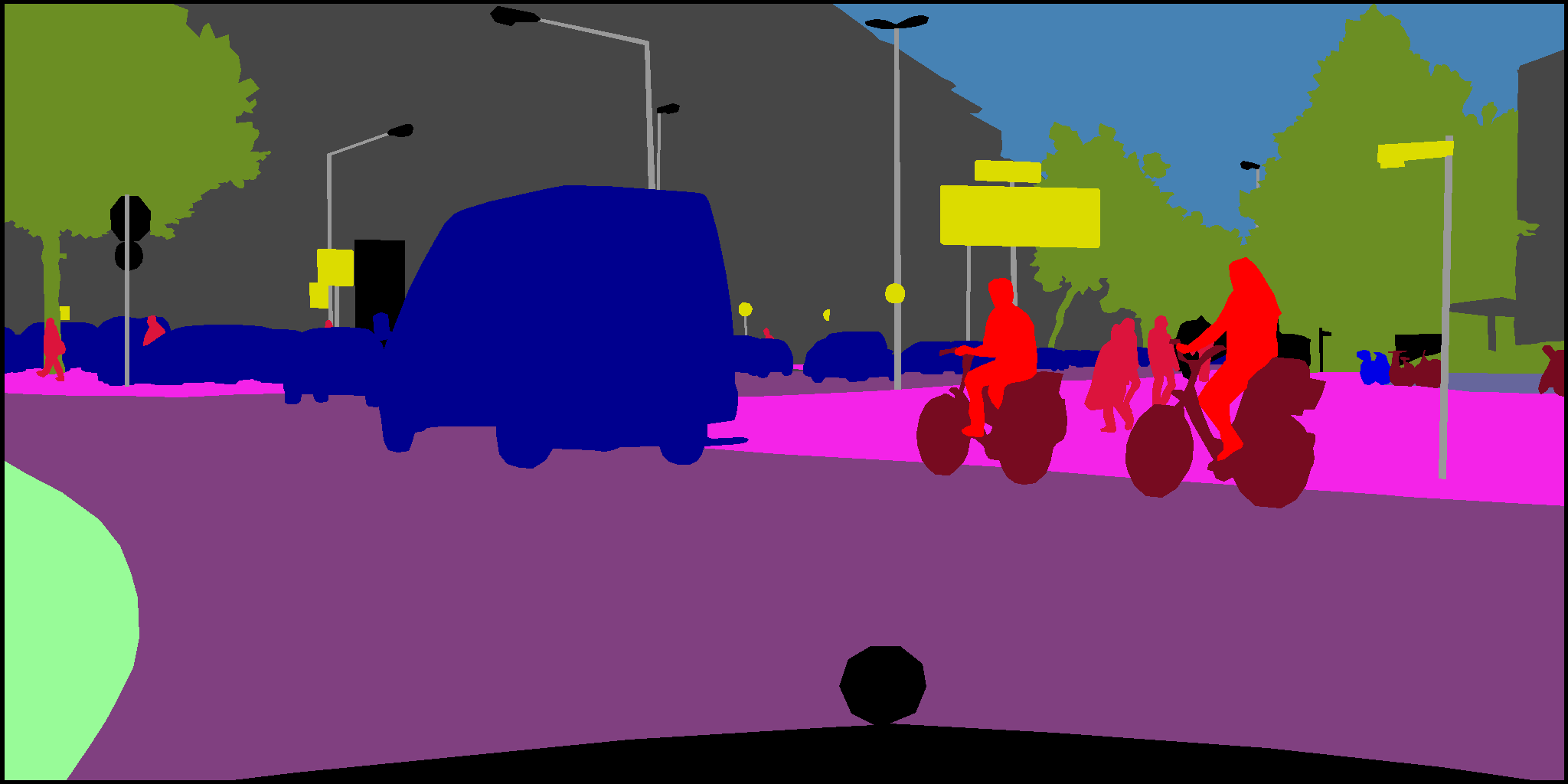}};
    \node at (20,-15.3){\includegraphics[width=0.15\textwidth]{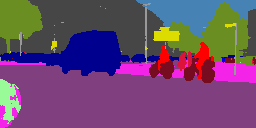}};
    \node at (30,-15.3){\includegraphics[width=0.15\textwidth]{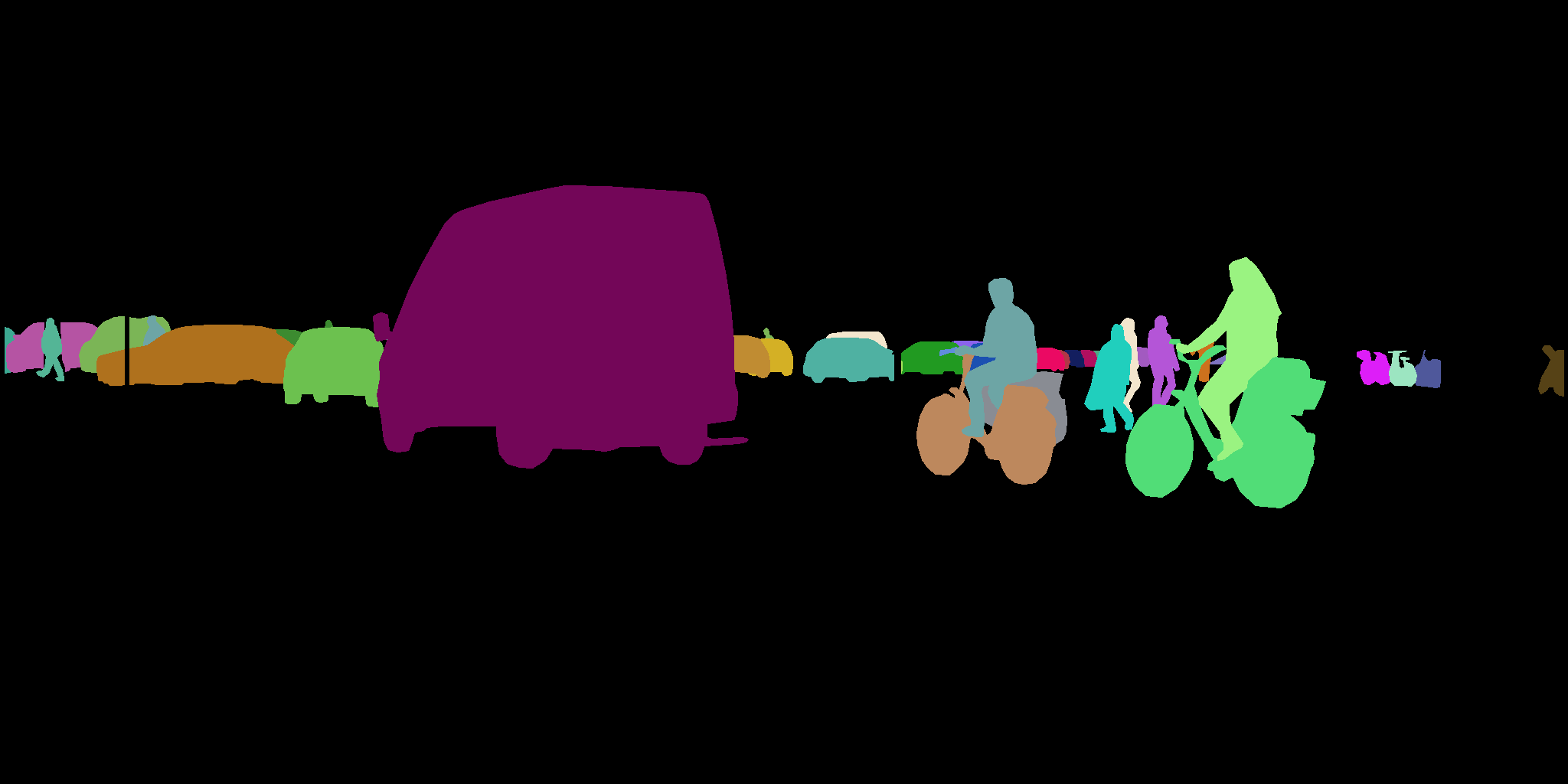}};
    \node at (40,-15.3){\includegraphics[width=0.15\textwidth]{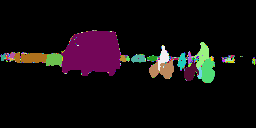}};
    \node at (50,-15.3){\includegraphics[width=0.15\textwidth]{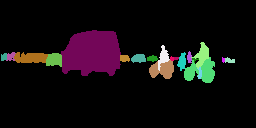}};

    \node at ( 0,-20.4){\includegraphics[width=0.15\textwidth]{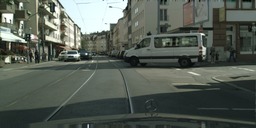}};
    \node at (10,-20.4){\includegraphics[width=0.15\textwidth]{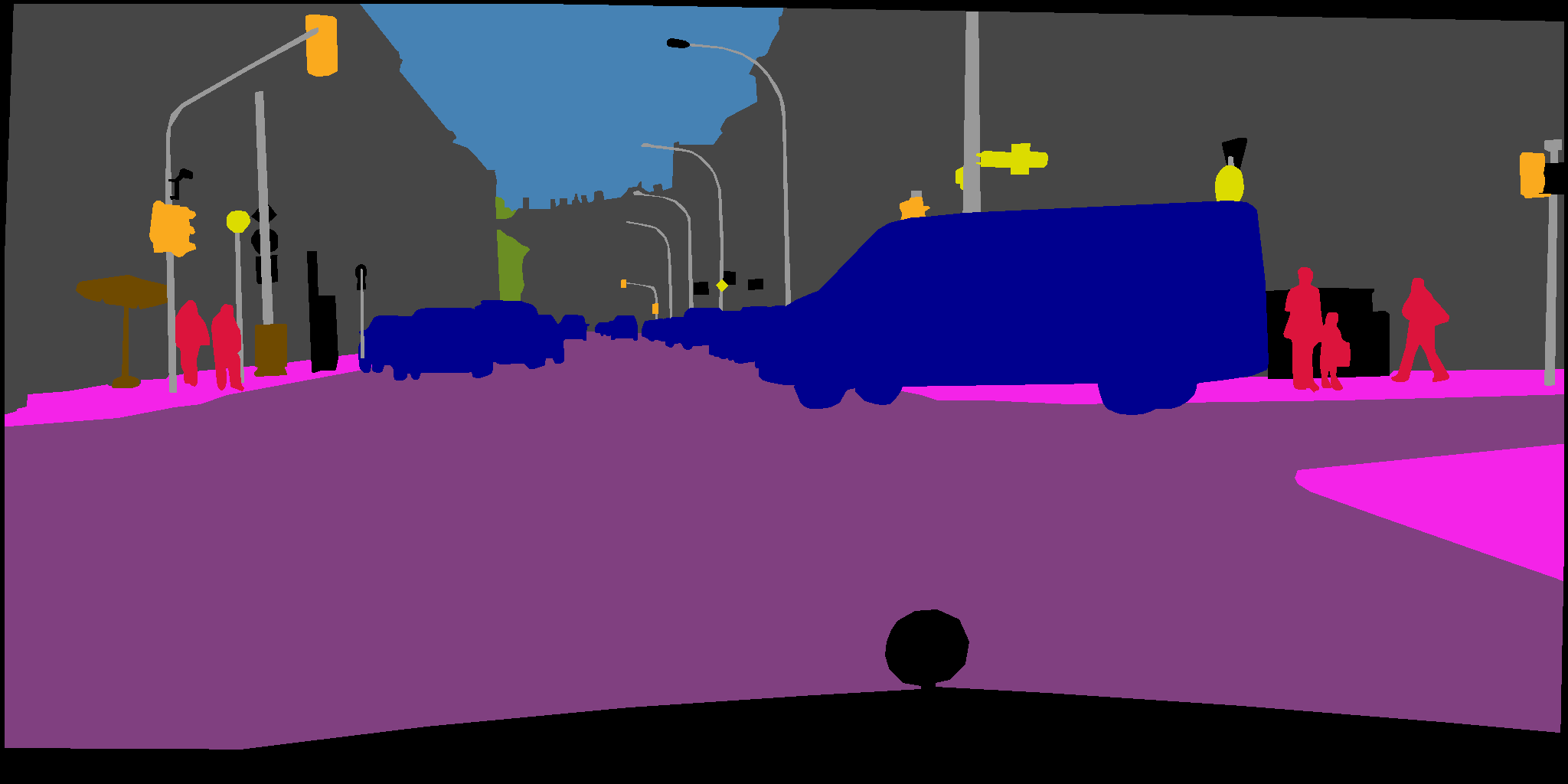}};
    \node at (20,-20.4){\includegraphics[width=0.15\textwidth]{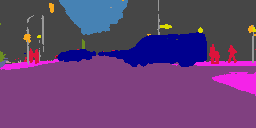}};
    \node at (30,-20.4){\includegraphics[width=0.15\textwidth]{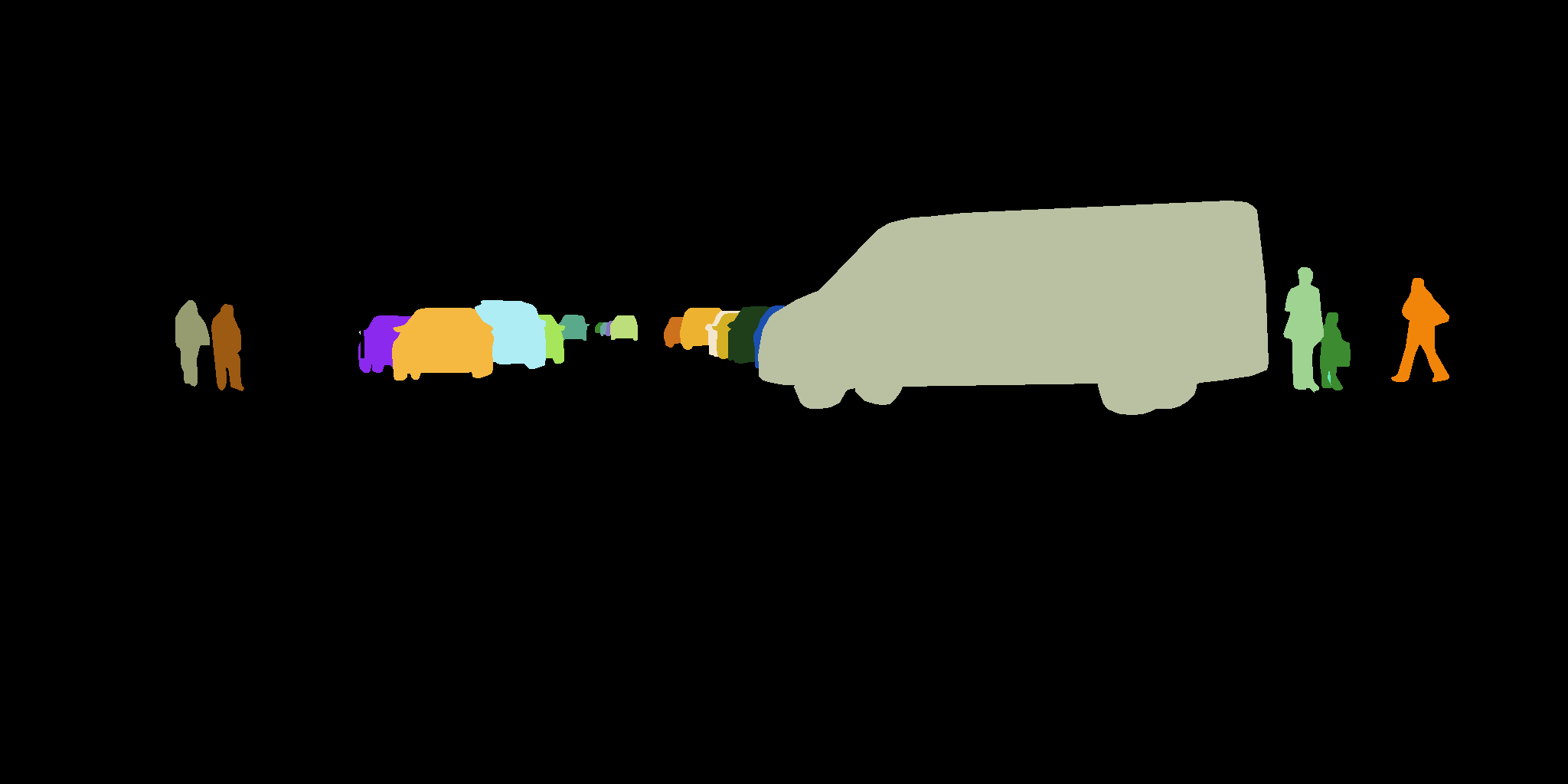}};
    \node at (40,-20.4){\includegraphics[width=0.15\textwidth]{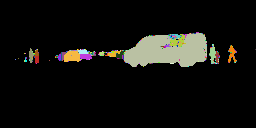}};
    \node at (50,-20.4){\includegraphics[width=0.15\textwidth]{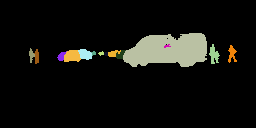}};

    \node at ( 0,-25.5){\includegraphics[width=0.15\textwidth]{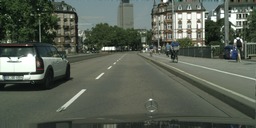}};
    \node at (10,-25.5){\includegraphics[width=0.15\textwidth]{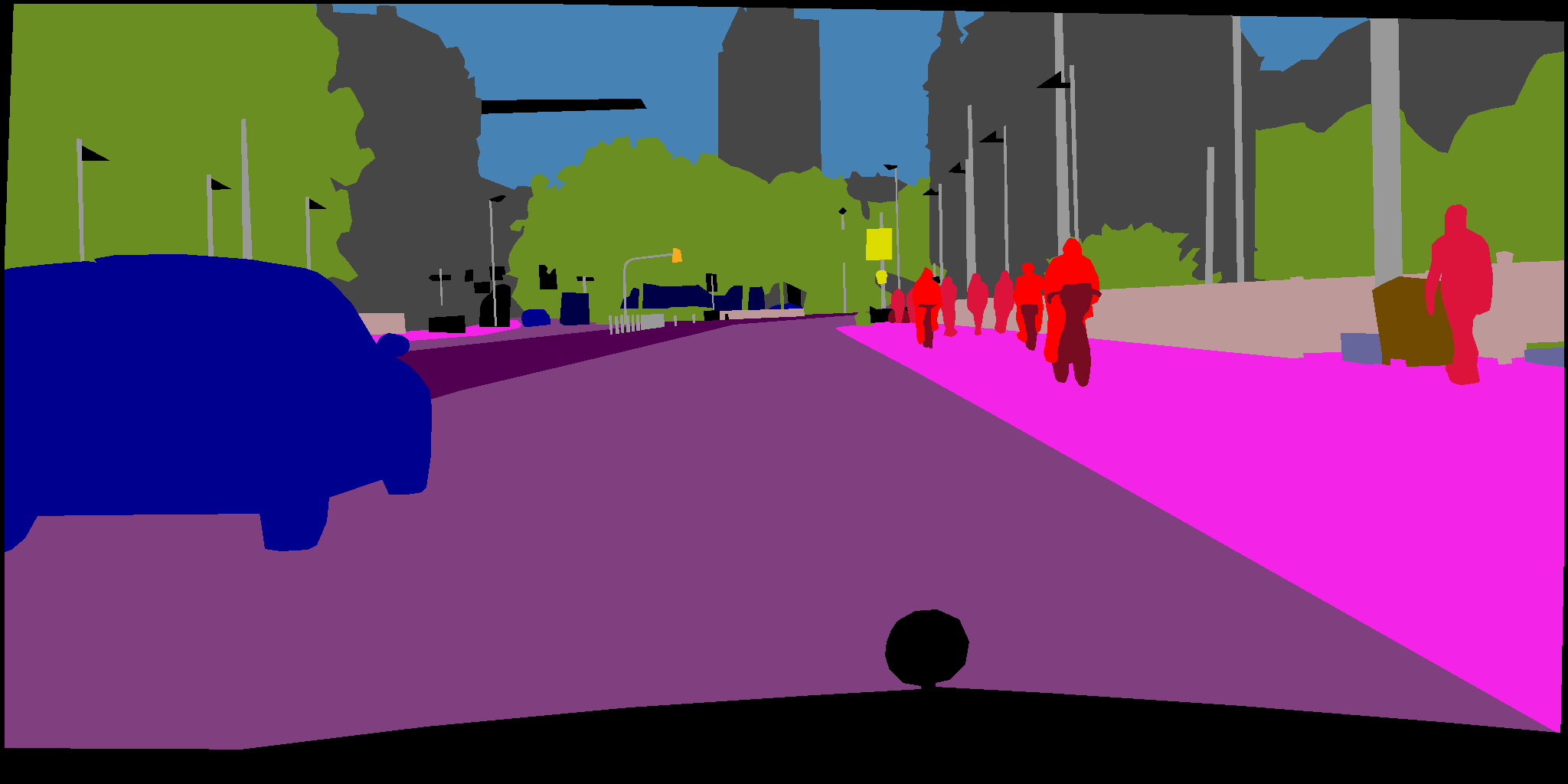}};
    \node at (20,-25.5){\includegraphics[width=0.15\textwidth]{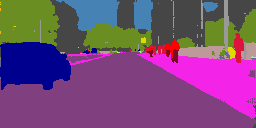}};
    \node at (30,-25.5){\includegraphics[width=0.15\textwidth]{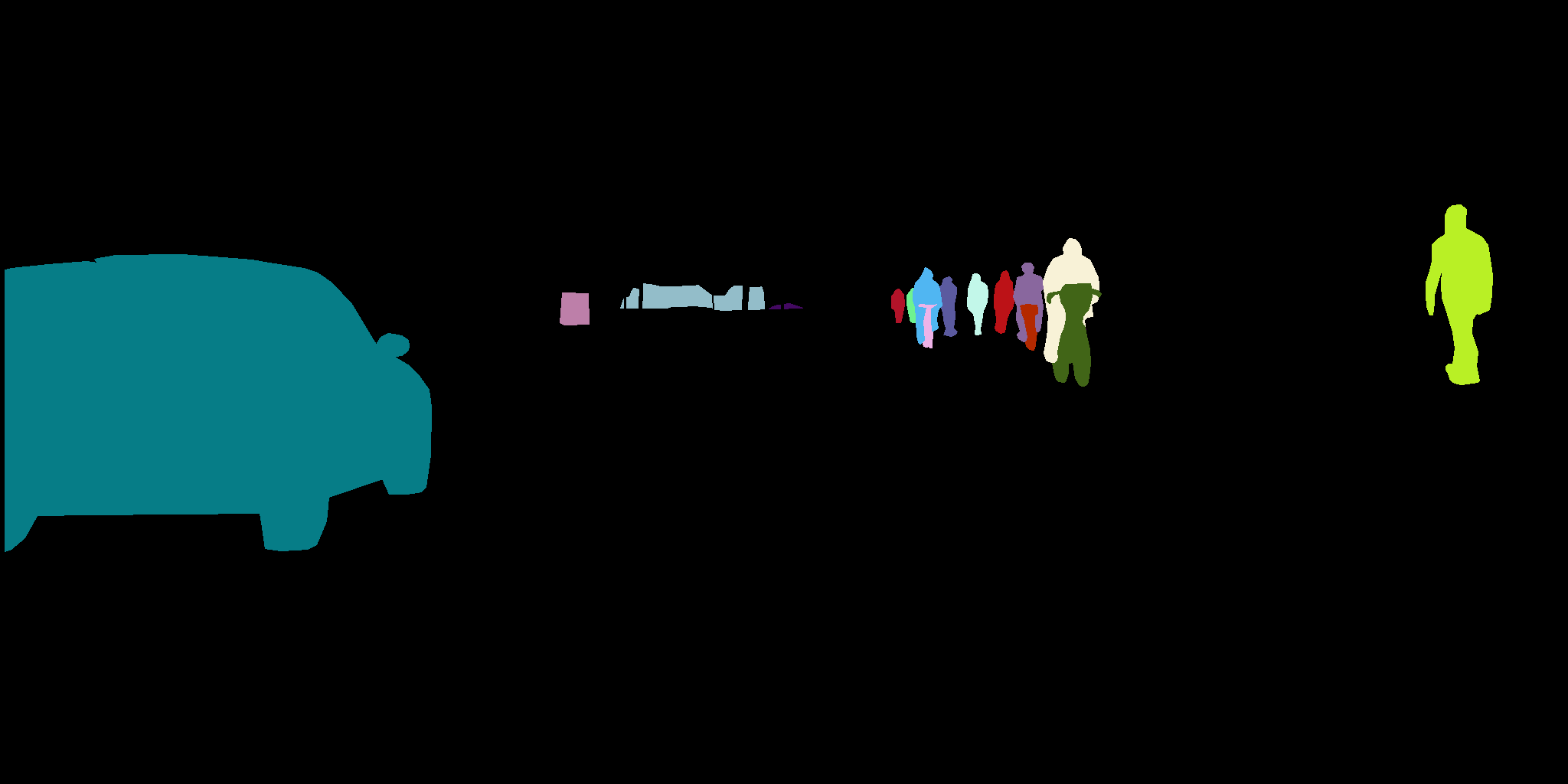}};
    \node at (40,-25.5){\includegraphics[width=0.15\textwidth]{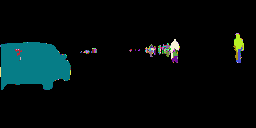}};
    \node at (50,-25.5){\includegraphics[width=0.15\textwidth]{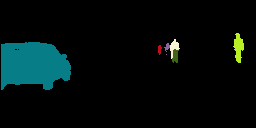}};

    \node at ( 0,-30.6){\includegraphics[width=0.15\textwidth]{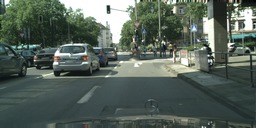}};
    \node at (10,-30.6){\includegraphics[width=0.15\textwidth]{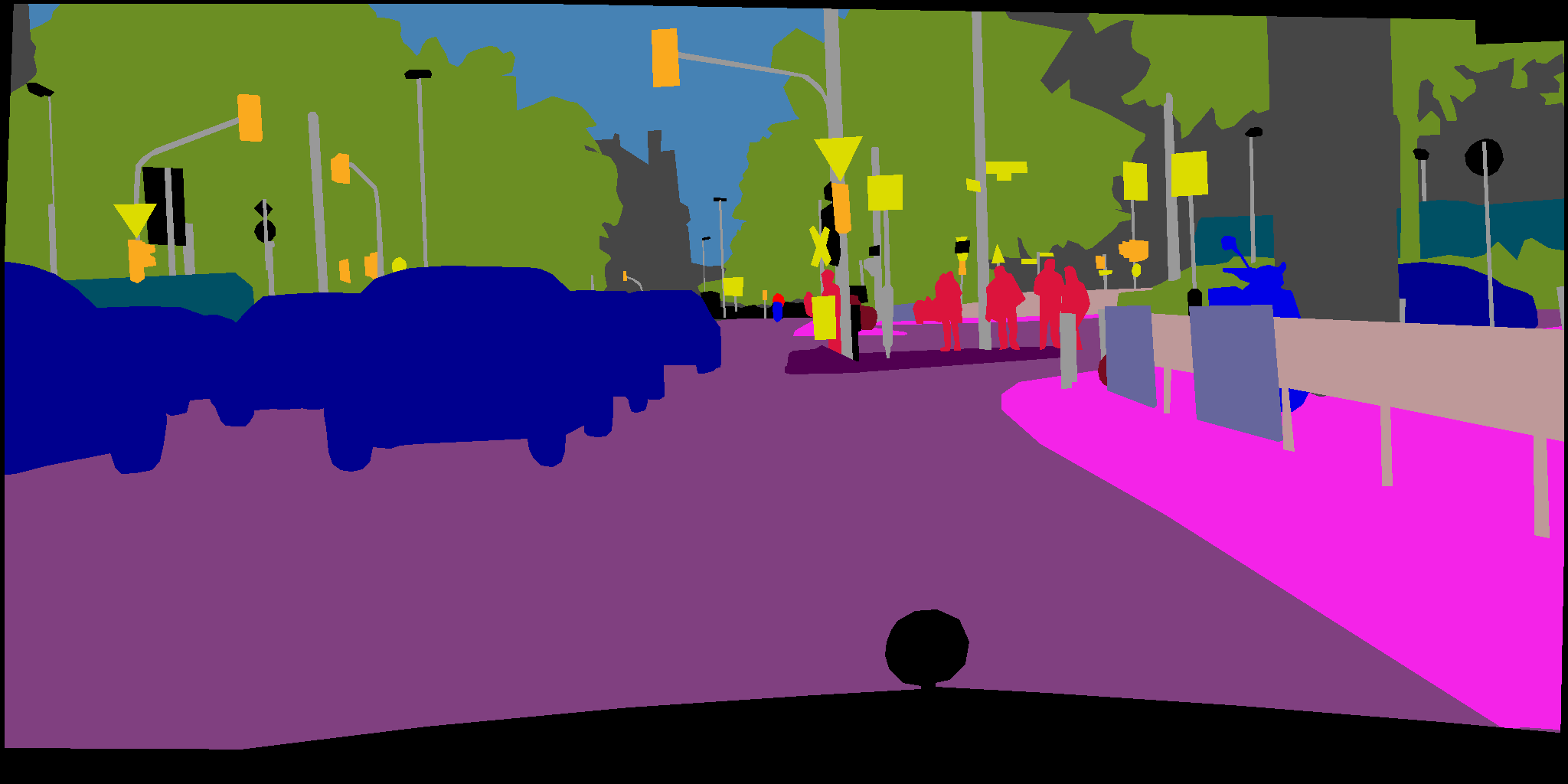}};
    \node at (20,-30.6){\includegraphics[width=0.15\textwidth]{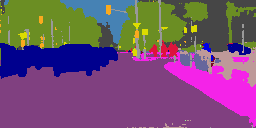}};
    \node at (30,-30.6){\includegraphics[width=0.15\textwidth]{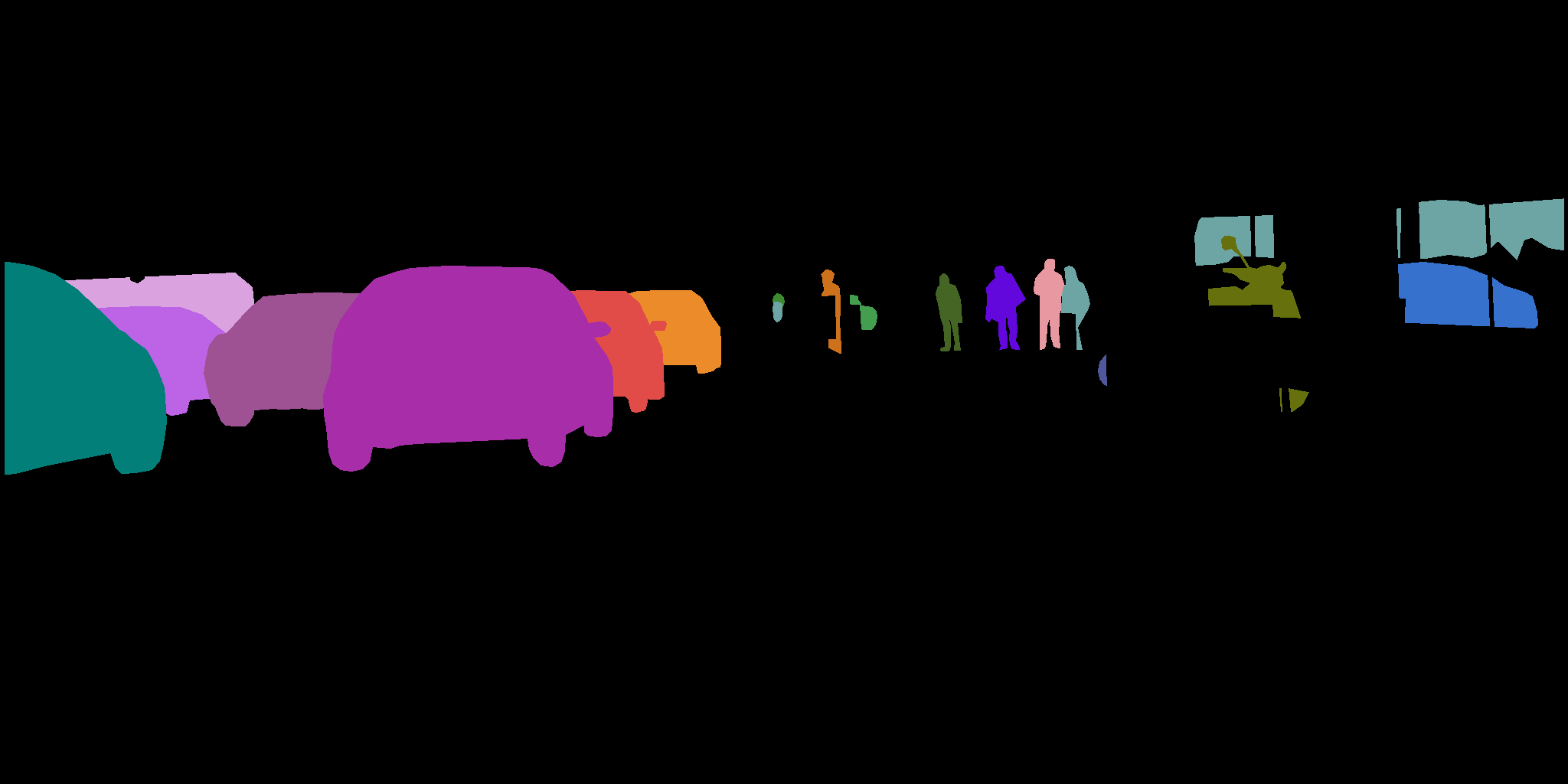}};
    \node at (40,-30.6){\includegraphics[width=0.15\textwidth]{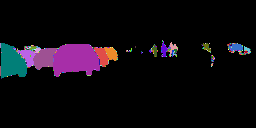}};
    \node at (50,-30.6){\includegraphics[width=0.15\textwidth]{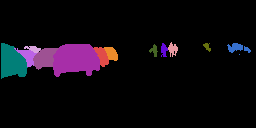}};

    \node at ( 0,-35.7){\includegraphics[width=0.15\textwidth]{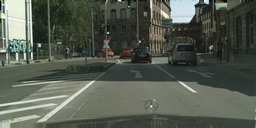}};
    \node at (10,-35.7){\includegraphics[width=0.15\textwidth]{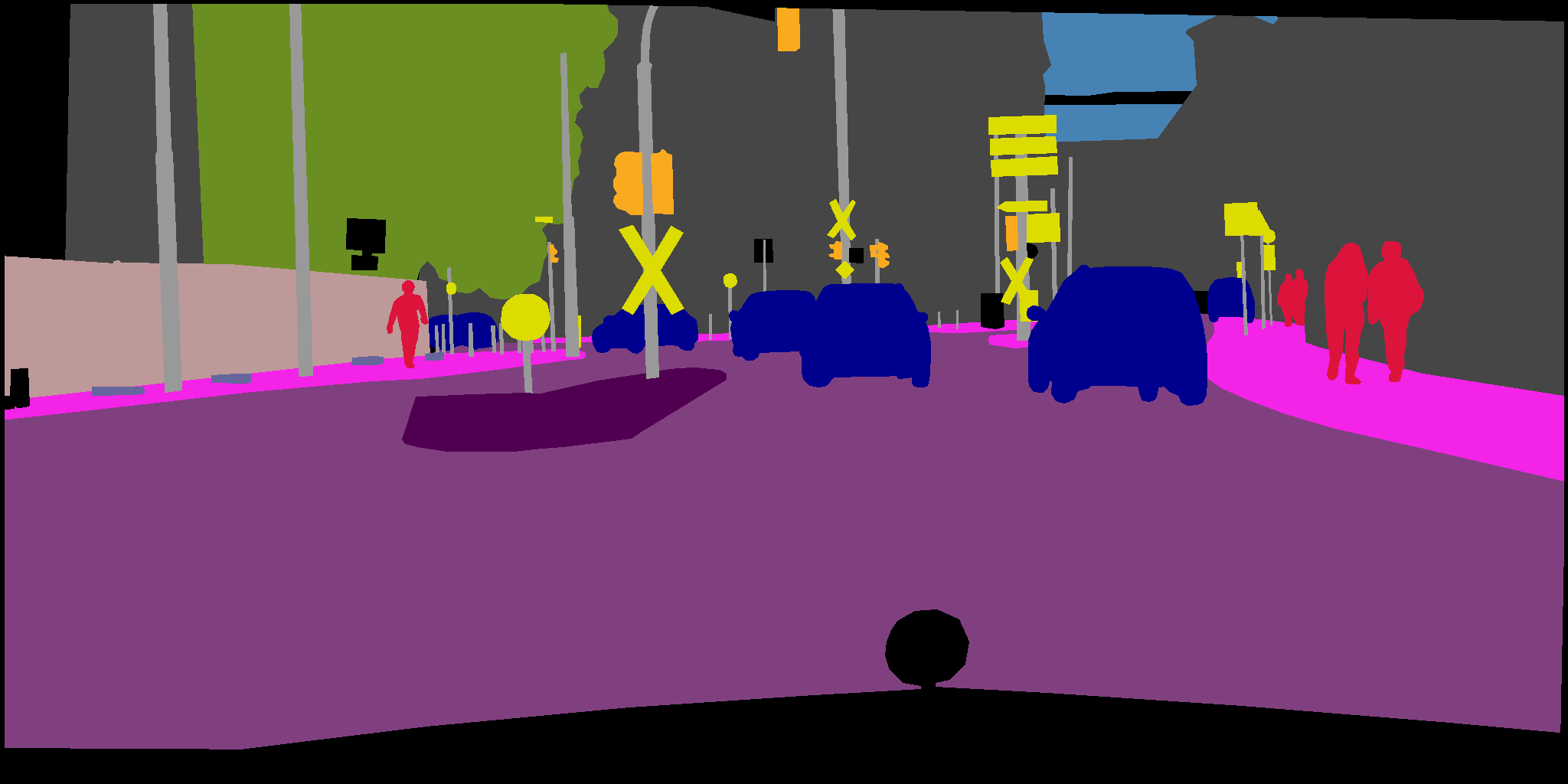}};
    \node at (20,-35.7){\includegraphics[width=0.15\textwidth]{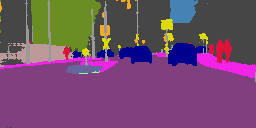}};
    \node at (30,-35.7){\includegraphics[width=0.15\textwidth]{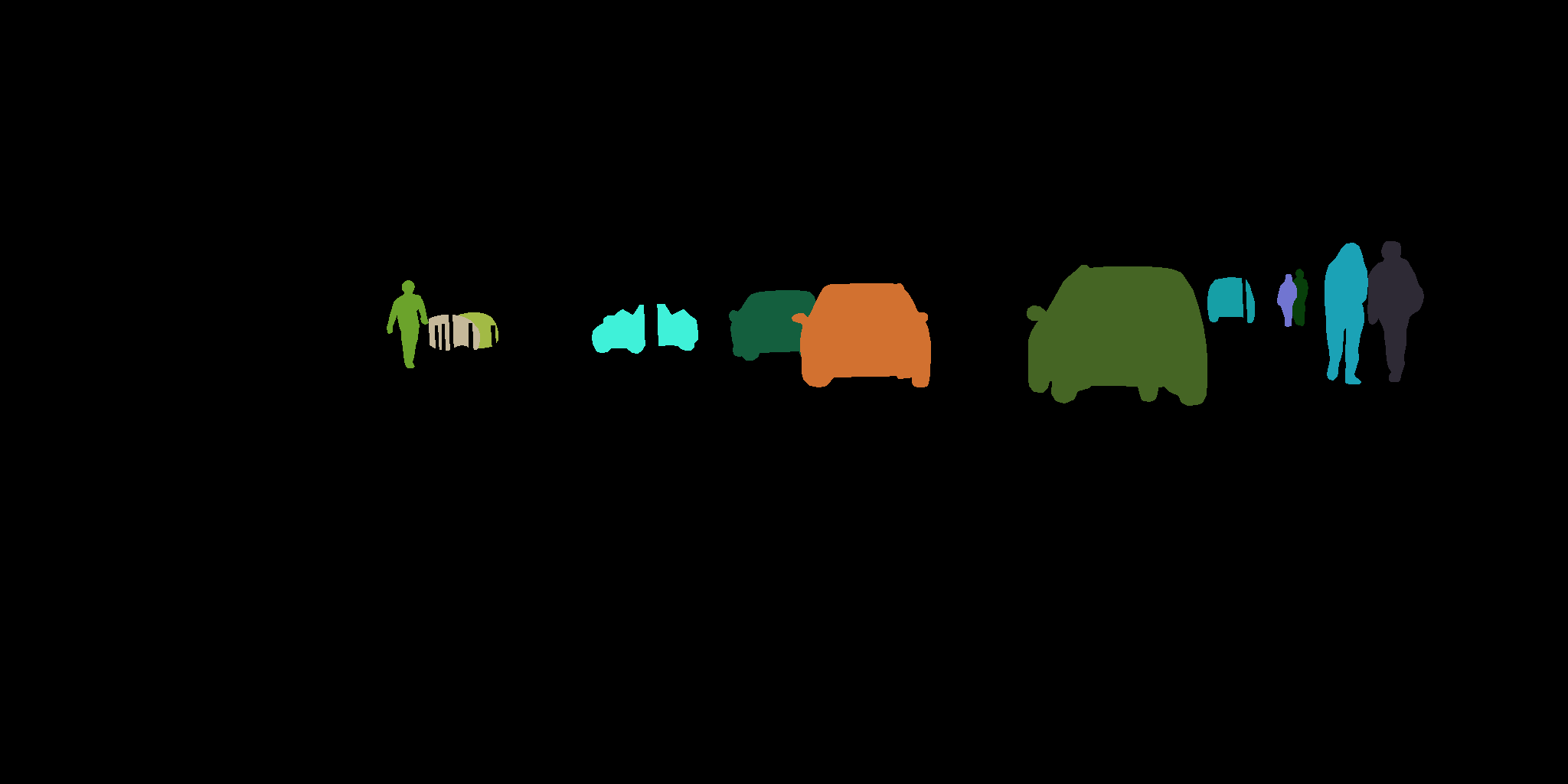}};
    \node at (40,-35.7){\includegraphics[width=0.15\textwidth]{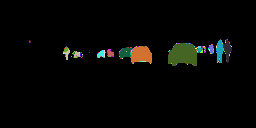}};
    \node at (50,-35.7){\includegraphics[width=0.15\textwidth]{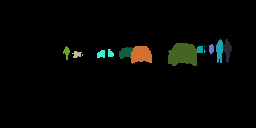}};
  \end{tikzpicture}
\end{center}
\caption{
Visualization of our predictions on the Cityscapes validation dataset~\cite{cityscapes16cvpr}, where we can compare with corresponding
ground truth (GT) and show respective RGB images.
}
\label{fig:csQualitativeResults}
\end{figure*}

\subsection{Outlook}

\begin{table}
\setlength{\tabcolsep}{4pt}
\renewcommand{\arraystretch}{1}
\centering
\small
\begin{tabular}{l@{\hskip 12pt}ccc}
\toprule
Algorithm                         & Dataset     & $\mapr$       & $\mapr^{50\%}$    \\
\midrule
Pixel Encoding \cite{uhrig16gcpr} & CS val      & $9.9$         & $22.5$            \\
Ours (\cite{uhrig16gcpr} scores)  & CS val      & $9.4$         & $22.1$            \\
Ours \alga{} (ResNet)             & CS val      & $11.3$        & \bst{26.8}        \\
Ours \algb{} (ResNet)             & CS val      & \bst{11.4}    & $26.1$            \\
\midrule
MCG+R-CNN \cite{cityscapes16cvpr} & CS test     & $4.6$         & $12.9$            \\
Pixel Encoding \cite{uhrig16gcpr} & CS test     & $8.9$         & $21.1$            \\
Ours \algb{} (ResNet)             & CS test     & \bst{9.8}     & \bst{23.2}        \\
\bottomrule
\end{tabular}
\caption{
Comparison of algorithms for instance segmentation on the Cityscapes (CS) dataset~\cite{cityscapes16cvpr}
using the mean average precision metrics introduced in \cite{cityscapes16cvpr}.
}
\label{tab:instance_results_CS_CS}
\end{table}

The reason for the varying performance for objects of different semantic classes certainly comes from their very different
typical forms, which we do not incorporate in our general approach. Uhrig et al. \cite{uhrig16gcpr} use different aspect
ratios for their sliding object templates to cope for these changes. In future work, we would like to combine multiple graphs
for different semantic classes to boost individual class performance. Also, the predicted FCN representation and scores
will be adjusted for better suiting the requirements of our graph optimization.

\section{Multiple Object Tracking}
\label{app:mot}

\subsection{Problem Statement}

Tang et al.~\cite{tang-2015} introduce a binary linear program w.r.t.~a graph $G = (V, E)$ whose nodes are candidate detections of humans visible in an image.
Every feasible solution is a pair $(x,y)$ with $x \in \{0,1\}^V$ and $y \in \{0,1\}^E$, constrained such that
\begin{align}
\forall \{v,w\} \in E: \quad 
    & y_{vw} \leq x_v \label{eq:supp_1} \\
    & y_{vw} \leq x_w \label{eq:supp_2} \\
\forall C \in \textnormal{cycles}(G)\ \forall e \in C: \quad
    & 1 - y_e \leq \hspace{-1.5ex} \sum_{f \in C \setminus \{e\}} \hspace{-1.5ex} (1 - y_f)
\end{align}
The objective function has the form below with coefficients $\alpha$ and $\beta$.
\begin{align}
\sum_{v \in V} \alpha_v x_v + \sum_{e \in E} \beta_e y_e 
\end{align}

We identify the solutions of this problem with the solutions of the \nllmp{} w.r.t.~the graphs $G' = G$, the label set $L = \{\epsilon, 1\}$
and the costs $c^{\not\sim} = 0$ and
\begin{align}
c_{vl} & := \begin{cases}
    \alpha_v & \textnormal{if}\ l = 1 \\
    0 & \textnormal{if}\ l = \epsilon \\
\end{cases} \\
c^\sim_{vw,ll'} & := \begin{cases}
    \beta_{vw} & \textnormal{if}\ l = 1 \wedge  l' = 1 \\
    0 & \textnormal{if}\ l = 1 \ \textnormal{xor}\ l' = 1 \\
    \infty & \textnormal{if}\ l = l' = \epsilon
\end{cases}
\enspace .
\end{align}
Note that in \cite{tang-2015}, $y_{dd'} = 1$ indicates a join.
In our \nllmp{}, $y_{dd'} = 1$ indicates a cut.

\subsection{Further Results}

A complete evaluation of our experimental results in terms of the Multiple Object Tracking Challenge 2016 can be found at \url{http://motchallenge.net/tracker/NLLMPa}.

\end{document}